\documentclass{article}
\usepackage{amsmath,amsfonts,amsthm,graphicx}
\usepackage{thm-restate}

\title{A Sharp KL-Convergence Analysis for Diffusion Models under Minimal Assumptions}

\usepackage[utf8]{inputenc} 

\pdfoutput=1

\usepackage{algorithm}
\usepackage{algorithmic}
\usepackage{amsmath}
\usepackage{amssymb}
\usepackage{amsthm,amsfonts}
\usepackage{array}
\usepackage{authblk}
\usepackage{babel}
\usepackage{booktabs}
\usepackage{caption}
\usepackage{colortbl}
\usepackage{enumitem}
\usepackage[T1]{fontenc}    
\usepackage{framed}
\usepackage{fullpage}
\usepackage{graphicx}
\usepackage[colorlinks, linkcolor=red, anchorcolor=blue, citecolor=blue]{hyperref}
\usepackage{indentfirst}
\usepackage[utf8]{inputenc}
\usepackage{lipsum}        
\usepackage{makecell}
\usepackage{mathrsfs}
\usepackage{mathtools}
\usepackage{xparse}

\usepackage[framemethod=default]{mdframed}
\usepackage{microtype}     
\usepackage{multicol}
\usepackage{multirow}
\usepackage[numbers]{natbib}
\usepackage{nicefrac}       
\usepackage{tablefootnote}
\usepackage{url}            
\usepackage{wrapfig,lipsum}
\usepackage{xcolor}         
\usepackage{xspace}

\usepackage{doi}
\usepackage{tikz}
\usetikzlibrary{arrows.meta}
\usetikzlibrary{decorations.pathreplacing}
\usepackage{amsmath,amsfonts,bm}

\def\1{\bm{1}}

\def\vw{{\bm{w}}}

\DeclareMathAlphabet{\mathsfit}{\encodingdefault}{\sfdefault}{m}{sl}
\SetMathAlphabet{\mathsfit}{bold}{\encodingdefault}{\sfdefault}{bx}{n}

\newcommand{\E}{\mathbb{E}}

\newcommand{\KL}[2]{\mathrm{KL}\left(#1 \big\| #2\right)}

\makeatletter
\def\thickhline{%
  \noalign{\ifnum0=`}\fi\hrule \@height \thickarrayrulewidth \futurelet
   \reserved@a\@xthickhline}
\def\@xthickhline{\ifx\reserved@a\thickhline
               \vskip\doublerulesep
               \vskip-\thickarrayrulewidth
             \fi
      \ifnum0=`{\fi}}
\makeatother

\newlength{\thickarrayrulewidth}
\newtheorem{lemma}{Lemma}[section]

\newtheorem{corollary}[lemma]{Corollary}

\newtheorem{assumption}{Assumption}[section]

\author{
  Nishant Jain\thanks{Mail to {nj27@illinois.edu}}   \quad
  Tong Zhang \\
  University of Illinois Urbana-Champaign \\
}

\begin{document}

\date{}

\maketitle

\begin{abstract}
    Diffusion-based generative models have emerged as highly effective methods for synthesizing high-quality samples. Recent works have focused on analyzing the convergence of their generation process with minimal assumptions, either through reverse SDEs or Probability Flow ODEs. The best known guarantees, without any smoothness assumptions, for the KL divergence so far achieve a linear dependence on the data dimension $d$ and an inverse quadratic dependence on $\varepsilon$. In this work, we present a refined analysis that improves the dependence on $\varepsilon$. We model the generation process as a composition of two steps: a reverse ODE step, followed by a smaller noising step along the forward process. This design leverages the fact that the ODE step enables control in Wasserstein-type error, which can then be converted into a KL divergence bound via noise addition, leading to a better dependence on the discretization step size. We further provide a novel analysis to achieve the linear $d$-dependence for the error due to discretizing this Probability Flow ODE in absence of any smoothness assumptions. We show that $\tilde{O}\left(\tfrac{d\log^{3/2}(\frac{1}{\delta})}{\varepsilon}\right)$ steps suffice to approximate the target distribution corrupted with Gaussian noise of variance $\delta$ within $O(\varepsilon^2)$ in KL divergence, improving upon the previous best result, requiring $\tilde{O}\left(\tfrac{d\log^2(\frac{1}{\delta})}{\varepsilon^2}\right)$ steps.
\end{abstract}

\section{Introduction}

\noindent
Recently, diffusion based models have picked up momentum for the trending generative modelling scenario. They are being extensively used  for image-generation \cite{song2019generative, croitoru2023diffusion, lugmayr2022repaint, song2020denoising, nichol2021glide, song2021maximum, ho2020denoising}, video-generation \cite{epstein2023diffusion, chen2023control}, semantic editing \cite{lugmayr2022repaint}, generating text \cite{li2022diffusion} or audio signals \cite{liu2023audioldm}, protein design analysis \cite{gruver2023protein, guo2024diffusion} and many other areas. The success of these diffusion models largely stems from their ability to generate high-quality samples using a denoising mechanism. This is achieved by defining a forward noising process that gradually perturbs data from the target distribution, during which a \textit{score} function is learned. New samples are then generated by iteratively applying the reverse of this process, guided by the learned score function. The forward process can be modeled as a Stochastic Differential Equation (SDE) \cite{song2020score}, and consequently the generation can be carried out by simulating its reverse-time SDE through discretization. Corresponding to this reverse-time SDE, there also exists a Probability Flow Ordinary Differential Equation (ODE) \cite{song2020score}, which shares the same marginal distributions at all times. Consequently, two main approaches have emerged for sample generation: simulating the reverse SDE \cite{song2020score, chen2022sampling} and simulating this Probability Flow ODE \cite{xu2024provably}. \\

\noindent
Several works \cite{chen2022sampling, chen2022improved, holzmuller2023convergence, benton2023nearly, li2024d} have targetted the theoretical underpinnings behind the working of these diffusion models, under various assumptions. These have established polynomial convergence w.r.t. data dimension under accurate score estimation, smoothness of true score at all times, bounded support assumptions. A recent focus \cite{chen2023improved, li2024d, benton2023nearly} is on further minimizing the assumptions and achieving good convergence just using the estimated score with error $\tilde{O}(\varepsilon)$. The best recent result \cite{li2024d} is achieving $O(d/T)$ convergence rate w.r.t TV distance for the DDPM diffusion model \cite{ho2020denoising}.
A recent work \cite{benton2023nearly} also achieved linear dependence on the data dimension for convergence w.r.t. KL divergence and requires $\tilde{O}(\frac{d}{\varepsilon^2})$ steps to achieve KL divergence within $\varepsilon^2$ w.r.t. to Gaussian perturbation of the true data distribution. Since TV is bounded by square root of the KL divergence, it is an important issue to investigate whether a better convergence rate is achievable in the KL-divergence. Even though the linear dependence on $d$ seems fine still the dependence on $\frac{1}{\varepsilon^2}$ might not be optimal. In this work, we are interested in investigating the following question: 

\begin{quote}
    \emph{Can we improve the dependence of the KL-divergence error on the $\varepsilon$ and thereby arrive at the improved convergence rate for diffusion models?}
\end{quote}

\noindent
To achieve this goal, we explore the line of works \cite{xu2024provably, gao2024convergence, li2024unified} which consider the reverse Probability Flow ODE. This is motivated from the finding that we can have better discretization dependence for each interval when analysing the Wasserstein type error directly using the ODEs \cite{chen2023probability}.
However, aggregating and bounding the error across all the intervals just based on this ODE requires additional assumptions either related to the error in Divergence \cite{li2024unified} or the Jacobi \cite{li2024d} between the true and approximated score. Therefore, \cite{chen2023probability} instead considers smoothness of true and approximate score function at all times to bound the Wasserstein error in each interval using the reverse ODE and adds a noising step utilizing Langevin dynamics to then convert the Wasserstein error to TV. However, this noising via Langevin finally results in a suboptimal dependence on $\varepsilon$. Given the improved dependence on step size the reverse Probability Flow ODE can offer, we also consider using it but instead of additional assumptions or the Langevin dynamics, we just consider taking a smaller step in the forward (noising) direction. 
This way we are able to consider the error due to discretization on the reverse ODE and then convert it into KL error via the noise addition with a better dependence on step size for each interval, which can then be aggregated across all the intervals. 
This ODE step and a smaller noise step can be interpreted as an alternative simulation of the reverse SDE based generation process. The idea of noise addition along the forward process to convert the Wasserstein type error to KL is inspired from a recent work \cite{jain2025multi} which provided convergence analysis for the consistency model framework \cite{song2023consistency}. \\
Unlike \cite{chen2023probability}, we consider the minimal assumptions scenario  \cite{chen2023improved, benton2023nearly} and just consider the accurate score estimation assumption. To achieve an optimal $d$-dependence under this setup for our considered ODE path, we take inspiration from \cite{benton2023nearly} (which considers the reverse SDE and by establishing equivalence to stochastic localization directly picks up a known result from the literature) and investigate additional relations between the score function and its Jacobian. Furthermore, as discussed in the paper, the analysis along this ODE path involves terms containing Laplacian of the score function along with the terms containing both score function and its Jacobian which makes the analysis more complicated than the SDE counterpart.
By establishing the required novel relations between the score function and its higher order gradient terms, we are able to achieve the linear dependence on data dimension $d$ for the ODE, matching the previous best result \cite{benton2023nearly}. Due to our improved dependence on discretization step size, this results in a new \textit{state-of-the-art} guarantee for the KL convergence requiring $\tilde{O}(d/\varepsilon)$ iterations to achieve $\varepsilon^2-$KL divergence improving up the previous best result of $\tilde{O}(\tfrac{d}{\varepsilon^2})$ \cite{benton2023nearly}. Also, since the TV distance is upper bounded by the square root of KL-divergence, this becomes the \textit{state-of-the-art} convergence guarantee for diffusion models as against the TV convergence guarantee provided in \cite{li2024d}.

\subsection{Related Work}
\noindent
Here we provide a literature review of the works targetting diffusion based generation broadly categorized into whether they consider the reverse SDE or the Probability Flow ODE. 
\paragraph{SDE-Based generation.} The effectivenss of this forward noising and the corresponding denoising process for generation was first majorly advocated by 
 Diffusion Probabilistic Models (DDPM) framework introduced \cite{ho2020denoising}, which utilised the Gaussian transition kernels at each step for noising and showed that the corresponding denoising kernels are also Gaussian, which can be estimated by denoising score matching during training \cite{ho2020denoising}. Going further it was shown that this forward noising in DDPMs can be seen as an SDE \cite{song2020score} and the generation process then correspond to the reverse SDE, if tractable. Since then there have been various works \cite{chen2022sampling,li2023towards, li2024towards, lee2022convergence} targetting the convergence of this generation process. To advocate for their usability in the real world, some recent works \cite{chen2023improved, benton2023nearly, li2024d} have also targetted setups for the SDE based generation methods requiring minimal assumptions (just the bound on score estimation during training via denoising score matching) and have achieved state-of-the-art convergence guarantees under various metrics. Specifically, \cite{benton2023nearly} shows only $O(d/\varepsilon^2)$ steps are required to be $\varepsilon^2$-close in KL w.r.t a Gaussian perturbation of the target distribution. On the other hand, \cite{li2024d} considers the TV-distance and shows $O(d/\varepsilon)$ steps are required to achieve $\varepsilon-$close TV of the perturbed data distribution.

\paragraph{ODE based generation.} The seminal work \cite{song2020score} which showed that the forward noising process in the diffusion based generation and the corresponding denoising process can be modelled as SDEs, also highlighted that there exists a Probability Flow ODE which achieves has the same marginal distribution at any time. It also advocated that this Probability Flow ODE can lead to faster sampling using the ODE solvers. Taking inspiration, a following work \cite{song2020denoising} then proposed a deterministic counterpart of the DDPM sampler and since then various works have attempted to investigate the convergence of these deterministic samplers \cite{li2024sharp, gao2024convergence, huang2025convergence, li2023towards, li2024unified} under various additional assumptions. The current best result \cite{li2024sharp} achieves a  TV distance of $\varepsilon$ (w.r.t. perturbation of the true data distribution) in $O(\frac{d}{\varepsilon})$ steps under score estimation and an additional assumption on the Jacobi of the estimated score. Another work \cite{li2024unified} requires a weaker assumption on the Divergence of the estimated score but achieves sub-optimal results. These works have also argued that under just the score estimation assumption, the TV-distance for these  deterministic samplers is lower bounded unlike SDE and thus, such additional assumptions are required. Another line of work is based on the predictor-corrector sampling \cite{song2020score, chen2023probability} which utilised both the ODE step and addition of small noise using langevin dynamics for smoothening the trajectory to avoid the error blow-up due to ODE. For this scenario, the convergence can be achieved \cite{chen2023probability} under standard assumptions on score estimation and the smoothness of the data distribution. Instead of utilising the Langevin dynamics for noise addition, we just directly add the stochastic noise along the forward process after the ODE step and are able to achieve the state-of-the-art convergence guarantee under just the score estimation assumption.

\section{Preliminaries and Setup}

\noindent
We now discuss the formulation behind diffusion models in detail, including both ODE and SDE-based generation. Following this, we discuss the assumptions used to achieve the results provided in the next section.

\paragraph{SDE considered and its discretization.} As discussed previously, diffusion models are based on a forward noising process and the reverse generation process. The forward process for $d-$dimensional setup can be seen as taking the given samples and gradually corrupting them using the SDE of the following form \cite{song2020score}:
\begin{align*}
    dx(t) = -\mu  (x(t),t)dt + g(t)dw_t
\end{align*}
where $x(0)=y\sim p_{data}$, $x(t)\in \mathbb{R}^d$, $\mu$ and $g$ correspond to the drift and diffusion coefficients, $w_t$ is the $d-$dimensional Brownian Motion. Following the popular choice of the OU process, we consider the following SDE:
\[
dx(t) = -x(t)dt + \sqrt{2}d{w}_{t}
\]
The corresponding OU process would be:
\begin{align}
\label{fwd_process}
x(t) = e^{-t}y + \sqrt{1-e^{-2t}}\cdot \epsilon, \qquad \epsilon\sim \mathcal{N}(0,I_d)
\end{align}
where $p_t$ denotes the law at time $t$, $x(t)\sim p_t$ and $y\sim p_{data}$. Also, the joint distribution of the random variables generated via this process at time-stamps corresponding to a sequence $\{t_1,..,t_K\}$: $(x_{t_1},...,x_{t_K})$ is denoted as $p_{t_1,...,t_K}$.
The resulting reverse SDE \cite{song2020score} for generation will be:
\begin{align}
\label{reverse_true_sde}
dx(t) = -x(t)dt - 2\nabla \ln p_{t}(x(t))dt + \sqrt{2}d\bar{w}_t   
\end{align}
where $\nabla \ln p_{t}(x(t))$ is referred to as the \textit{score} function. If the forward process is run from time $T$ then initializing from $p_T$ and going along this reverse process for a time $T-t$ will result in the marginal $p_t$.
For the corresponding probability flow ODE \cite{song2020score}, we have the following equation:
\begin{equation}
dx(t) = -x(t)dt - s(t,x(t))dt
\label{eq:fode}
\end{equation}
where we denote $s(t,\cdot)=\nabla \ln p_t(\cdot)$. Using the Exponential Integrator discretization \cite{chen2023improved} where we divide the overall generation time into intervals $[t_i,t_{i+1}]$ and fix the input to the score function for each interval to be the value at the start ($x_{{k}}$), leads to the following ODE for the interval $[t_{k-1},t_{k}]$:
 \begin{align*}
 dx(t) = -x(t)dt -s_{}(t_{k}, x_{{k}})dt
\end{align*}
\paragraph{Empirical Counterpart.} Practically, we donot have the true score function $s(t,\cdot)$ and instead during training it is approximated via \textit{denoising score matching} \cite{song2020denoising}. Denoting that approximated score function as $\hat{s}(t,\cdot)$, we have the following empirical version (discretized and using the approximate score) of the true ODE:
\begin{align}
\label{discretized_em_ode}
 d\hat{x}(t) = -\hat{x}(t) dt -\hat{s}({t_k},\hat{x}_{k})dt
\end{align}
where we denote the law of this empirical process at time $t$ as $\hat{p}_t$. For any particular discretization $\{t_k\}^N_{k=1}$ of the reverse process, the joint distribution for  the true ($x_{t_1},...,x_{t_N}$), reverse ($\hat{x}_{t_1},....,\hat{x}_{t_N}$) processes are denoted as $p_{t_1,...., t_N}$, $\hat{p}_{t_1,...,t_N}$. Also the conditional distribution at $t-1$ conditioned on $t$ is denoted as $p_{t-1|t}$, $\hat{p}_{t-1,t}$ for the true and empirical processes respectively. 
We now discuss the assumptions used in our theoretical framework. \\

\paragraph{Assumptions.} As discussed in the introduction, for our theoretical analysis, we take inspiration from the line of works operating under minimal assumptions \cite{benton2023nearly, chen2023improved, li2024d}, and just use the following standard assumptions: 
\begin{assumption}
\label{assumption_score_est}
    \textit{For the discretization sequence $\{t_{k}\}^{K+1}_{k=1}$ discussed in the next section (and used in the Inference Algorithm \ref{alg:1}), the score function estimate} $\{\hat{s}(t,\cdot)\}_{1\leq t \leq T}$ satisfies:
    \begin{equation}
    \label{eq_score_est}
     \frac{1}{T}\sum_{k=1}^{K+1}h_k\mathbb{E}_{x \sim p_{t_k}}\left [\|\hat{s}(t_k,x)- s(t_k,x)\|^2 \right ] \leq \varepsilon^2_{score}.
\end{equation}
where $h_k=t_{k}-t_{k-1}$ corresponds to the step size of the discretization.

\end{assumption}

\begin{assumption}
\label{finite_moment}
    The data distribution $p_{data}$ has finite second order moment $\mathbb{E}_{x_0 \sim p_{\text{data}}} \left[ \|x_0\|_2^2 \right] = m_2 < \infty$.
\end{assumption}

\subsection{Notations}
\label{notations}
\noindent
 As discussed above, $y$ corresponds to the data distribution $p_{data}$, $x(t)$ (with its law denoted by $p_t$ and score function as $s(\cdot)$) corresponds to the forward OU process and $z(t)$ (with its law denoted by $q_t$ and score function as $s_r(\cdot)$) is a rescaled version of the forward process discussed in Appendix. $\tilde{x}(t)$ corresponds to the discretized version of the true reverse process. $\hat{x}'_k$ denotes the sequence of random variables generated by our algorithm for a  discretization sequence $\{t_k\}$ and  their law is denoted by $\hat{p}_{t_k}$. The step size $t_k-t_{k-1}$ is denoted as $h_k$. $x_k$ corresponds to the random variables generated by the forward process for this time sequence. $\tilde{x}_{k-1}$, $\hat{x}_{k-1}$ corresponds to random variables generated by running our proposed scheme (with True, Empirical Probability Flow ODE respectively) for a single interval $[t_{k-1},t_k]$ starting from $x_k$ at $t_k$. 
 $\tilde{x}_{k-0.5}$,$\hat{x}_{k-0.5}$ corresponds to the random variable generated by taking two steps along the discretized (Empirical, True respectively) probability flow ODE in reverse direction  starting from ${x}_{k}$ and $x_{k-0.5}$ denotes two steps of Probability Flow ODE from $x_k$.
$\nabla s(t,x)$ denotes the Jacobian of the score and $\partial_t$ corresponds to the partial derivative w.r.t. time $t$. We further define $\partial_i$ as the partial derivative w.r.t. $i^{th}$ coordinate of the spatial variable x/z (can be interpreted as $\partial_{x_i}/\partial_{z_i}$). We also define Laplacian operator $\Delta = \sum_{i=1}^d\partial_i\partial_i$. The $i^{th}$ element of the score vector $s(\cdot)$ is denoted by $s(\cdot)_i$. $\mathcal{N}(0,I_d)$ denotes the d-dimensional standard Normal distribution.
For two terms $P, Q$ $P\lesssim Q$ means there exist an absolute constant $C_1$ such that $P \leq C_1 Q$.

\section{Main Results}
\noindent
{As discussed in previous works \cite{chen2023probability}, based on the ODE, we can just analyze the Wasserstein error by using the Young's and Grownwall's inequality, and the aggregation leads to blow-up in the error and thus, some noise is added  via Langevin dynamics based corrector after a small ODE step to instead convert into the TV and then aggregate it. Here, we instead take a more simplistic perspective for the diffusion use-case, inspired from a recent work \cite{jain2025multi}. We consider first taking a step along the reverse ODE to bound the Wasserstein distance and then taking a partial step in the noising (forward) direction to convert the Wasserstein to KL. \\
\noindent
We now consider a discretization sequence $0<\delta=t_0<t_1< t_2 <...<t_K<t_{K+1}=T$, where  $T$ denotes the total time for the generation process initializing from $\mathcal{N}(0,I_d)$. Denoting $h_{k} = t_{k}-t_{k-1}$, we provide the inference procedure in Algorithm \ref{alg:1}. It is based on the Exponential Integrator discretization of the Empirical Probability Flow ODE (Eq. \ref{discretized_em_ode}) in the step 4 followed by noise addition along the forward process in step 6. The step along the ODE can be used to control the Wasserstein error \cite{chen2023probability} and then the noise addition can convert this into KL. This is discussed further in the next section and in detail (along with technical lemmas) in Appendix \ref{wass_kl_subsec}. We denote the generated sequence from our algorithm by random variables $\hat{x}'_k$ (corresponding to time $t_k$) and their law  by $\hat{p}_{t_k}$ and the joint distribution for the complete sequence as $\hat{p}_{t_1,...,t_K, t_{K+1}}$. Similarly, corresponding to the sequence generated by the forward process along these time-stamps, we will have the join distribution as ${p}_{t_1,...,t_K,t_{K+1}}$.
Figure \ref{fig_alg} shows this algorithm/generation process. }
\noindent
\begin{algorithm}[H]
\caption{Inference Algorithm For Diffusion Models}
\label{alg:1}
\begin{algorithmic}[1]
\STATE \textbf{Given:} Discretizing sequence $\{t_1,..,t_K\}$, $h_k=t_k-t_{k-1}$, $\hat{p}'_{t_K}$ as the Normal distribution $\mathcal{N}(0,I_d)$
\STATE Sample $\hat{x}'_K \sim \hat{p}_{t_K}$
\FOR{$k = K, K-1, \ldots, 1$}
    \STATE $\hat{x}'_{k-0.5} =  e^{ h_k +h_{k-1}} \hat{x}'_{k} +
 (e^{h_k+h_{k-1}}-1) 
 \hat{s}(t_{k},\hat{x}'_{{k}})$ \\
    \STATE Sample $\eta_k \sim \mathcal{N}(0, I_d)$
    \STATE $\hat{x}'_{k-1} = e^{-h_{k-1}} \hat{x}'_{k-0.5} + \sqrt{1 - e^{-2h_{k-1}}}\, \eta_k$ 
\ENDFOR
\STATE \textbf{Output} $\hat{x}'_1$
\STATE $\hat{p}_{t_1}$ denotes the density of $\hat{x}'_1$
\STATE $\hat{p}'_{t_1,.., t_{K+1}}$ denotes joint density of $(\hat{x}'_1, \hat{x}'_2, \ldots, \hat{x}'_{K+1})$
\end{algorithmic}
\end{algorithm}
\noindent
{We now provide the guarantee for the distribution generated by Algorithm \ref{alg:1} in terms of KL divergence w.r.t. perturbation of the true data distribution (law of the forward process at $t_1$) $p_{t_1}$. This is the early stopping scenario (since $t_1>\delta>0$)\cite{benton2023nearly, chen2023improved} as we have not used any assumption regarding the smoothness of the data distribution.}

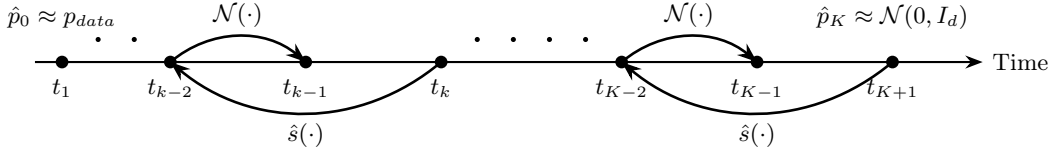
\begin{figure}
\centering
\begin{tikzpicture}[scale=1.2, >=Stealth, 
    thick, 
    every node/.style={font=\small}
  ]

  \draw[->] (-0.5,0) -- (10,0) node[right] {Time};

  \node at (-0.2,0.5) {$\hat{p}_0 \approx p_{data}$};
  \fill (-0.2,0) circle (2pt) node[below=3pt] {$t_{1}$};
  \fill (0.2,0.25) circle (0.8pt) node[below=3pt] {};
  \fill (0.6,0.25) circle (0.8pt) node[below=3pt] {};
  \fill (1,0) circle (2pt) node[below=3pt] {$t_{k-2}$};
  \fill (2.5,0) circle (2pt) node[below=3pt] {$t_{k-1}$};
  \fill (4,0) circle (2pt) node[below=3pt] {$t_k$};
  \fill (4.4,0.25) circle (0.8pt) node[below=3pt] {};
  \fill (4.8,0.25) circle (0.8pt) node[below=3pt] {};
  \fill (5.2,0.25) circle (0.8pt) node[below=3pt] {};
  \fill (5.6,0.25) circle (0.8pt) node[below=3pt] {};
  \fill (6,0) circle (2pt) node[below=3pt] {$t_{K-2}$};
  \fill (7.5,0) circle (2pt) node[below=3pt] {$t_{K-1}$};
  \fill (9,0) circle (2pt) node[below=3pt] {$t_{K+1}$};
  \node at (9,0.5) {$\hat{p}_K \approx \mathcal{N}(0,I_d)$};

  \draw[->, bend left=40, line width=1pt] (4,0) to node[below] {$\hat{s}(\cdot)$} (1,0); 
  \draw[->, bend left=40, line width=1pt] (1,0) to node[above] {$\mathcal{N}(\cdot)$ } (2.5,0);
  \draw[->, bend left=40, line width=1pt] (9,0) to node[below] {$\hat{s}(\cdot)$} (6,0); 
  \draw[->, bend left=40, line width=1pt] (6,0) to node[above] {$\mathcal{N}(\cdot )$ } (7.5,0);

\end{tikzpicture}
\caption{Demonstrating the two updates: (a) along the generation process using $\hat{s}(\cdot)$ and (b) the forward noising process $(\mathcal{N}(\cdot))$, of our proposed scheme.} 
\label{fig_alg}

\end{figure}

\begin{restatable}{theorem}{maintheorem}
\label{main_theorem_non_smooth}
    Under assumptions \ref{assumption_score_est} and \ref{finite_moment}, for the generation process provided in Algorithm \ref{alg:1}
    given the discretization sequence $0<\delta=t_0<t_1<t_2<...<t_{K+1}=T$, with step size denoted $h_k = t_k-t_{k-1}$. For $T\geq 1$, $K > d(\log(\frac{1}{\delta}) + T)$ there will exist a $c$  such that when $h_k = c \min\{1,t_k\}$, we will have (denoting $\hat{p}_{t_k}$ as the marginal at $t_k$ for this algorithm, and $p_t$ as the distribution of the forward process (Eq. \ref{fwd_process}) at time $t$):
     \begin{align}
     \label{main_thm_eq}
         \KL{p_{t_1}}{\hat{p}_{t_1}}
         &\lesssim (d+m_2)e^{-T} + d^2c^3K + T \varepsilon^2_{score} 
    \end{align}
where recall $x\lesssim y$ means there is some absolute constant $C$ such that $x \leq Cy$.
\end{restatable}

\noindent
We provide the proof of this theorem in the Appendix (\ref{final_theorem_proof}) and discuss a sketch of the complete proof in the next section. The first term corresponds to the error due to initializing the algorithm from $\mathcal{N}(0,I_d)$, second term corresponds to the error due to discretization and the third term is due to error in score estimation (Assumption \ref{assumption_score_est}). 
From the definition of $c$, it can be observed that (discussed in the proof as well) $c^3$ should be $O\left(\frac{(\log \frac{1}{\delta}+T)^3}{K^3}\right)$. Also we can observe that $T$ is required to just have a logarithmic dependence on $d$ and thus, the second term corresponding to the discretization error will be $\tilde{O}\left(\frac{d^2}{K^2}\right)$. 
We formalize this in the following corollary discussing the iteration complexity.

\begin{corollary}
\label{main_corollary}
    Under assumptions \ref{assumption_score_est}, \ref{finite_moment} running the update scheme in Algorithm \ref{alg:1} for the SDE based generation via diffusion models for a total time $T=\log\left(\frac{d+m_2}{\varepsilon_{score}}\right)$ with an exponentially decaying step size sequence $h_k = t_k-t_{k-1} = c\min\{t_k,1\}$ where $c=\Theta\left(\frac{\log(\frac{1}{\delta})+T}{K}\right)$  achieves a $\mathrm{KL}$-divergence error of $\tilde{O}(\varepsilon^2_{score})$ with an iteration complexity $K= \Theta\left(\frac{d\left(\log (\frac{1}{\delta})^{3/2}\right)}{\varepsilon_{score}}\right)$, improving upon the previous best complexity \cite{benton2023nearly} of $\Theta\left(\frac{d \log^2(\frac{1}{\delta})}{\varepsilon^2_{score}}\right)$ for the $\mathrm{KL}$-divergence.
\end{corollary}

\section{Proof Sketch}
\noindent
We now provide a brief sketch of the proof for Theorem \ref{main_theorem_non_smooth} and the complete details are provided in the Appendix. We begin by first discussing the decomposition of KL divergence into the Wasserstein type error aggregated in each interval. Since the ODE can result in a better dependence on the discretization step size \cite{chen2023probability}, this serves as the main motivation of our Algorithm \ref{alg:1} proposed with the aim of improving the step size dependence on the KL. Then, we discuss bounding the discretization error along this ODE path in the non-smooth scenario. Finally, we discuss on how the optimal dependence on $d$ can be achieved for this non-smooth setup leading to state of the art convergence guarantee for the KL divergence.
 
{
\paragraph{KL control for diffusion via Wasserstein type error.} We can first decompose the KL between between the generation process and the forward process at $t_1$: $\KL{p_{t_1}}{\hat{p}_{t_1}}$ using the data processing inequality and chain rule as follows  (Lemma \ref{lemma_second_KL_data_processing_ineq}):
\begin{align*}
   \KL{p_{t_1}}{\hat{p}_{t_1}} \leq  \KL{{p}_{t_{K+1}}}{\hat{p}_{t_{K+}}} + 
   \E_{{p}_{t_1,..,t_{K+1}}} \left[ \sum_{k=2}^{K+1}
       \KL{{p}_{t_{k-1}|t_k}(\cdot|{x}_{k})}{\hat{p}_{t_{k-1}|t_k}(\cdot|{x}_k)}\right]
\end{align*}
where $p_{t_{k-1}|t_k}$ denotes the conditional distribution of the true process at $t_{k-1}$ given $x_k$ at $t_k$ and similarly $\hat{p}_{t_{k-1}|t_k}$ for the generation process. 
The first term on the RHS is just the initialization error (the error by using the Normal distribution for initialization as against the distribution of the forward process after time $T$) and can be bounded following previous works \cite{chen2022sampling, chen2023improved} as $(d+m_2)e^{-T}$. The second term denotes the summation of the KL error aggregated in each interval $[t_{k-1},t_k]$ when the true and the generation process start from the same point ($x_k$). Now, to calculate this term, we will consider the following update for the interval $[t_{k-1},t_{k}]$ starting from $x_k$ using the empirical ODE and noise:
\begin{align}
  \hat{x}_{k-0.5} = &
 e^{h_k + h_{k-1}} x_{k} +
 (e^{h_k + h_{k-1}}-1) 
 \hat{s}(t_{k},{x}_{{k}}) \label{eq_empirical_gen_first}\\
  \hat{x}_{k-1} 
  =& e^{-h_{k-1}} \hat{x}_{k-0.5} 
 + \sqrt{1-e^{-2 h_{k-1}}} \epsilon_k , \qquad \epsilon_k \sim N(0,I). \label{eq_empirical_gen_second}
\end{align}
Based on this, the $\KL{{p}_{t_{k-1}|t_k}(\cdot|{x}_{k})}{\hat{p}_{t_{k-1}|t_k}(\cdot|{x}_k)}$ term can be written as (Lemma \ref{wass_kl_subsec}):
\begin{align*}
    \KL{{p}_{t_{k-1}|t_k}(\cdot|{x}_{k})}{\hat{p}_{t_{k-1}|t_k}(\cdot|{x}_k)} =  e^{-2h_{{k-1}}}
    \frac{\|
      x_{k-0.5}-\hat{x}_{k-0.5}\|_2^2}
    {2 (1  -e^{-2h_{k-1}})}    
\end{align*}
where $x_{k-0.5}$ denotes the random variable on the true reverse process at time $t_k-h_k - h_{k-1}$. This is inspired from a recent work \cite{jain2025multi} which instead considered the multi-step sampling scheme for ODE based consistency models \cite{song2023consistency}.
}
\subsection{Bounding $\E\|x_{k-0.5}-\hat{x}_{k-0.5}\|^2$}
\noindent
{For error control of this term, we define an additional process for the interval $[t_{k-1},t_k]$ (starting from $x_k$ and governed by Exponential Integrator discretization of the true Probability Flow ODE): $\tilde{x}_{k}$:
 \begin{align}
  \tilde{x}_{k-0.5} = &
 e^{h_k + h_{k-1}} x_{k} +
 (e^{h_k + h_{k-1}}-1) 
 s(t_{k},x_{{k}}) \label{discretized_true_first}\\
  \tilde{x}_{k-1} 
  =& e^{-h_{k-1}} \tilde{x}_{k-0.5} 
 + \sqrt{1-e^{-2 h_{k-1}}} \epsilon_k , \quad \epsilon_k \sim N(0,I) \label{discretized_true_second} .
\end{align}}

\noindent
Now, we decompose the target term corresponding to our scheme $\mathbb{E} \|\hat{x}_{k-0.5}-{x}_{k-0.5}\|_2^2 $ for each interval as follows:
\begin{align}
\label{error_decom}
    \sqrt{\mathbb{E}[\|x_{k-0.5}- \hat{x}_{k-0.5}\|^2_2]} \leq \underbrace{\sqrt{\mathbb{E}[\|x_{k-0.5}- \tilde{x}_{k-0.5}\|^2_2]}}_{T_{\text{dis}}} + \underbrace{\sqrt{\mathbb{E}[\|\tilde{x}_{k-0.5}- \hat{x}_{k-0.5}\|^2_2]}}_{T_{\text{est}}}
\end{align}
where $T_{est}$ is the error due to score estimation and $T_{dis}$ is the error due to the discretization of the true process. This score estimation error term can be written as $(e^{h_k+h_{k-1}}-1)^2\mathbb{E}[\|s(t_k,x_{k})- \hat{s}(t_k, x_{k}\|^2]$ (Lemma \ref{overall_w_2_error_decomp}) and aggregated across all the intervals can be bounded as ${O}(h_k\varepsilon^2_{score})$ using Assumption \ref{assumption_score_est} (further discussed in the proof of Theorem \ref{main_theorem_non_smooth} in Section \ref{final_theorem_proof}). To bound the discretization error, we first define a rescaled version of the original process as $z(t) = e^tx(t)$ (Section \ref{rescaling_and_discretizaton_error_analysis}) with the law at time denoted by $q_t$, score function denoted as $s_r(t,z(t))$. 
Using the ODE path of Eq. \ref{eq:fode} and the Integral Remainder form of the Taylor Expansion, we bound this discretization error (Lemma \ref{discretization_error_rescaled}):
\begin{align*}
\mathbb{E}\left[\|z_{k-0.5}-\tilde{z}_{k-0.5}\|_2^2\right] \leq \frac{1}{2}(h_k+h_{k-1})^3  \int_{t_{k-2}}^{t_k}e^{4t}\mathbb{E}\left[\left\|s'_r(t,z(t)) \right\|_2^2\right]dt
\end{align*}
where the derivative of the score $s_r'(t,z(t))$ can be calculated using Taylor Theorem as $s_r'(t,z(t)) = \frac{d}{dt}s_r(t_,z(t)) = {\frac{\partial s_r(t,z)}{\partial t}}
+ {{\frac{\partial s_r(t,z)}{\partial z}}{\frac{d z(t)}{d t}} \Big|_{z=z(t)}}$. {This is different from previous works \cite{benton2023nearly, chen2023improved} which instead consider the reverse SDE in Eq. \ref{reverse_true_sde} and thereby resulting in discretization error contribution to the KL for each interval as: $\int^{t_{k}}_{t_{k-1}}\mathbb{E}\left[\|s(t_k,x_k) - s(t,x(t))\|^2\right]dt$, which has a worse dependence on $h_k$ ($O(h^2_k)$) and bounded using the Jacobian of the score. Specifically, the best bound is achieved in \cite{benton2023nearly} which directly expresses the derivative of $\mathbb{E}\left[\|s(t_k,x_k) - s(t,x(t))\|^2\right]$ term w.r.t. $t$ in terms of $\mathbb{E}_{p_t}\left[\|\nabla s(t,x)\|\right]$ and then integrates the bound on this term to bound the discretization error.

\noindent
{Now, we discuss in detail on how to bound $\int^{t_k}_{t_{k-2}}\mathbb{E}\left[\left\|s_r'(t,z(t)) \right\|_2^2\right]dt$ following Eq \ref{eq:fode}. Using the trick proposed in \cite{chen2023improved}, we can write $z(t)$ as a Gaussian perturbation in $y\sim p_{data}$, thereby rewriting score function at time $t$: $s_r(t,z)$ as $\mathbb{E}_{y|z}\frac{y-z}{1-e^{-2t}}$ (Lemma \ref{score_calc}). A straightforward option is to then calculate the Jacobian $(\nabla s_r(t,z)$, partial derivative w.r.t. time $(\partial_t s_r(t,z))$ and use the Taylor Theorem to get an expression for $\mathbb{E}\left[\left\|s_r'(t,z(t)) \right\|_2^2\right]$. This can then be upper bounded using the fact that $\frac{y-z}{\sqrt{e^{2t}-1}}\sim \mathcal{N}(0,I_d)$ \cite{chen2023improved} (also discussed in Lemma \ref{score_bound}). This leads to a better dependence on $h_k$ than \cite{benton2023nearly} but will lead to $d^3$ dependence of our target term $\left(\int_{t_{k-2}}^{t_{k}}\mathbb{E}\left[\left\|s'(t,x(t)) \right\|_2^2\right]dt\right)$ and thus, a $d^{3/2}$ dependence for $\KL{p_{t_1}}{\hat{p}_{t_1}}$, which is worse than the $d$-dependence achieved for KL in \cite{benton2023nearly}.}\\

\subsubsection{Achieving the optimal $d$-dependence for ODE}

\noindent
 Lemma \ref{score_bound} shows that $\mathbb{E}\left[\left\|s_r(t,z)\right\|^2\right]$ can be bounded as $O(\frac{d}{e^{2t}-1})$ as against $O(\frac{d^2}{(e^{2t}-1)^2})$ for $\mathbb{E}\left[\left\|\nabla s_r(t,z)\right\|_F^2\right]$. {Therefore to have the $d-$dependence, we take inspiration from \cite{benton2023nearly} which first  establishes the equivalence of reverse SDE based to Stochastic Localization and then exploits a well known result from the Stochastic Localization literature (Lemma 1 in the paper). As we are considering the ODE path, instead of directly utilising such result, we begin by first establishing $\frac{d}{dt}\mathbb{E}_{q_t}[\|s_r(t,z)\|^2] = -2e^{2t}\mathbb{E}_{q_t}[\|\nabla s_r(t,z)\|^2_F]$ in Lemma \ref{spatial_time_relation}. Since our target term for the discretization error $\left(\int^{t_k}_{t_{k-2}}\mathbb{E}_{q_t}\left[\left\|s_r'(t,z(t)) \right\|_2^2\right]dt\right)$ depends on the integral (w.r.t. time) of the Jacobian term, using $\frac{d}{dt}\mathbb{E}_{q_t}[\|s_r(t,z)\|^2]$ can improve the $d^2$ contribution from this term to $d$. This serves as the motivation for the remaining sketch. As discussed above, since \cite{benton2023nearly} involves the reverse SDE, it just requires bounding $\mathbb{E}_{q_t}[\|\nabla s_r(t,z)\|^2_F]$ term. However for our considered Probability Flow ODE path, as discussed above, we need to bound the overall derivative term: $\mathbb{E}\left[\left\|s'_r(t,z)\right\|^2\right]$ which includes partial derivative w.r.t. time $\partial_t s_r(t,z)$ making the analysis much more complicated as discussed next.} \\

\noindent
We first convert the time-derivative to spatial derivatives using the the Fokker Plank equation for the score function (Lemma \ref{score_fpe}):
\begin{align*}
    \partial_t s_r(t,z) = e^{2t} \Delta s_r(t,z) + 2e^{2t}\nabla s_r(t,z)^\top s_r(t,z)
\end{align*}
where recall from Section \ref{notations} that $\Delta$ denotes the Laplacian of the score $s_r$.
This, results in the overall derivative term being represented only in terms of spatial derivative as follows (Lemma \ref{overall_to_spatial}):
\begin{align}
\label{overall_derivative_eq_proof_sketch}
    \mathbb{E}_{q_t}\left[\|s_r'(t,z)\|^2\right] &= e^{4t}{\mathbb{E}_{q_t}\left[{ \| \Delta s_r(t,z) \|^2_2} +{\left\|\nabla s_r(t,z)^\top s_r(t,z)\right\|_2^2} + (\Delta s_r(t,z))^\top\left(\nabla s_r(t,z)^\top s_r(t,z)\right)\right]} 
\end{align}
Since this overall derivative involves a term containing both score, its Jacobian and a term containing the Laplacian of the score, bounding this involves more complex analysis as compared to the SDE scenario \cite{benton2023nearly}. Now, based on the motivation discussed above to achieve optimal $d$ dependence by expressing \\
$\frac{d}{dt}\mathbb{E}_{q_t}[\|s_r(t,z)\|^2] = -2e^{2t}\mathbb{E}_{q_t}[\|\nabla s_r(t,z)\|^2_F]$, we further establish similar relations of the RHS terms in Eq. \ref{overall_derivative_eq_proof_sketch}. For the term comprising both $s_r$ and $\nabla s_r$ in Eq. \ref{overall_derivative_eq_proof_sketch}, we provide the generalized version of Lemma \ref{spatial_time_relation}
which considers general power $m$ in $\frac{d}{dt}\mathbb{E}_{q_t}\left[\|s_r(t,z)\|_2^m\right]$ and the term $\mathbb{E}\left[\|\nabla s_r(t,z)\|^2_2\right]$ (Lemma \ref{generalized_lemma}):
\begin{align*}
    e^{-2t}\frac{d}{dt}\mathbb{E}[\| s_r(t,z)\|_2^m] = -m\mathbb{E}[\|s_r(t,z)\|^{m-2}_2\|\nabla s_r(t,z)\|^2_F] 
    -\frac{m(m-2)}{4} \mathbb{E}\left[\|s_r(t,z)\|^{m-4}_2\left\| \left(\nabla \|s_r(t,z)\|^2_2\right)\right\|^2_2\right]
\end{align*}
We then utilise this equation to first write the second term in the RHS of our main Eq. \ref{overall_derivative_eq_proof_sketch} in terms of $\frac{d}{dt}\mathbb{E}_{q_t}[\| s_r(t,z)\|_2^q]$. To target the first term, we establish another novel relation by starting with the term $\frac{d}{dt}\mathbb{E}_{q_t}[\|\nabla s_r(t,z)\|_2^2]$ and expressing it in terms of $\int \Delta {q_t(z)}\|\nabla s_r(t,z)\|_F^2dt$, $\mathbb{E}_{q_t}[\| \Delta s_r(t,z)\|^2$ and \\
$\mathbb{E}_{q_t}[\|  \nabla \|s_r(t,z)\|^2_2\|_2^2$. Then, we rearrage and express $\mathbb{E}_{q_t}[\| \Delta s_r(t,z)\|^2$ in terms of $\int \Delta {q_t(z)}\|\nabla s_r(t,z)\|_F^2dt$, $\mathbb{E}_{q_t}\left[\|\nabla \|s_r(t,z)\|^2_2\|_2^2\right]$ and $\frac{d}{dt}\mathbb{E}_{q_t}\left[\|\nabla s_r(t,z)\|^2\right]$ (Lemma \ref{Laplacian_score}). We bound the term $\int \Delta {q_t(z)}\|\nabla s_r(t,z)\|_F^2dt$ as follows (Lemma \ref{hess_p_grad_s_norm}, $C_d$ present in the Lemma statement is $O(1)$ since $C_d \leq 12$ for $d\geq 10$):
\begin{align*}
   \int \Delta {q_t(z)}\|\nabla s_r(t,z)\|_F^2dt \lesssim  \frac{d^2}{(e^{2t}-1)^{3}}    -  \frac{e^{-2t}d}{(e^{2t}-1)}\frac{d}{dt}\mathbb{E}_{q_t}[\|s_r(t,z)\|^2] 
\end{align*}
leading to an overall bound on the $\mathbb{E}_{q_t}[\| \Delta s_r(t,z)\|^2]$ as (Lemma \ref{Laplacian_score}):
 \begin{align*}
       \mathbb{E}_{q_t}\left[\|\Delta s_r(t,z)\|_2^2\right] \lesssim \frac{d^2}{(e^{2t}-1)^{3}}    -  \frac{de^{-2t}}{(e^{2t}-1)} \frac{d}{dt}\mathbb{E}_{q_t}[\|s_r(t,z)\|^2]  - e^{-2t}\left(\frac{d}{dt}\mathbb{E}_{q_t}\left[\|\nabla s_r(t,z)\|_F^2\right] +\frac{d}{dt}  \mathbb{E}_{q_t}\left[\|s_r(t,z)\|^4_2\right]\right)
    \end{align*}
Finally, this leads to the following bound on $\mathbb{E}_{q_t}[\|  s_r'(t,z)\|_2^2]$ (Lemma \ref{overall_derivative_bound}):
\begin{align*}
        \mathbb{E}_{q_t}[\|s_r'(t,z)\|_2^2] \lesssim { {\frac{d^2e^{4t}}{(e^{2t}-1)^{3}}    -  \frac{e^{2t}d}{(e^{2t}-1)} \frac{d}{dt}\mathbb{E}_{q_t}[\|s_r(t,z)\|^2]  -  e^{2t}\left(\frac{d}{dt}\mathbb{E}_{q_t}\left[\|\nabla s_r(t,z)\|_F^2\right] + \frac{d}{dt}\mathbb{E}_{q_t}[\|s_r(t,z)\|^4]\right)} } 
\end{align*}
Integrating this and summing up across all the intervals, choosing $h_k=c \min\{t_k,1\}$ following the previous works  \cite{chen2023improved, benton2023nearly}   and scaling back to $\tilde{x}(t)$ along with accounting for the score estimation error and the initialization error leads to the following final expression for $\KL{p_{t_1}}{\hat{p}_{t_1}}$ (section \ref{final_theorem_proof}, refer to the analysis there for more details):
 \begin{align*}
         \KL{p_{t_1}}{\hat{p}_{t_1}}
         &\lesssim (d+m_2)e^{-T} + d^2c^3K + T \varepsilon^2_{score} 
    \end{align*}
    where due to the exponentially decaying step size $c \lesssim \frac{\log(\frac{1}{\delta})+T}{K}$ which results in $K = \Theta\left(\frac{d\log^{3/2}(\frac{T}{\delta})}{\varepsilon} \right)$ to achieve $\tilde{O}(\varepsilon^2)$ $\KL{p_{t_1}}{\hat{p}_{t_1}}$ error. 
}

\section{Conclusion}
{In this work we provided an improved analysis for generation process of the diffusion models under just the $L^2$-accurate score estimation and finite second moment of the data distribution assumption. We showed that by modelling the SDE based generation process as an ODE step followed by noising  and thereby targetting the discretization error along this ODE path can lead to better dependence on the step size. We also introduced a novel analysis framework for this ODE path which expresses the overall derivative of the score function in terms of spatial derivatives and establishes relations between the score and its first, second order spatial derivatives. This resulted in achieving linear dependence on $d$ for the considered ODE path, leading to a new \textit{state-of-the-art} convergence guarantee for KL divergence. Since KL upper bounds the square of the TV-distance by Pinsker's inequality, our result also provides a stronger guarantee than the best existing result for the TV distance \cite{li2024d}. 
An interesting future direction can be to investigate if the dependence on the step size can be improved further when considering this ODE step followed by noising framework, thereby enhancing the dependence on $\varepsilon$ and achieving faster convergence.}

\section*{Acknowledgment}
\noindent
This work is partially supported by NSF grant No. 2416897 and ONR grant No. N000142512318.

{\small
\bibliographystyle{plain}
\bibliography{egbib}
}

\newpage
\tableofcontents
\appendix

\newpage

\section{Proof of Theorem \ref{main_theorem_non_smooth}}

\subsection{Bounding $\KL{p_{t_1}}{\hat{p}_{t_1}}$ as aggregation of $\mathbb{E}[\|x_{k-0.5}- \hat{x}_{k-0.5}\|^2_2]$ for each interval}
\label{wass_kl_subsec}

\noindent
We now discuss two lemmas: a) The first one converts Wasserstein type error between the empirical and true process to KL for each interval and the second one aggregates the KL across all the intervals. \\

\begin{lemma}
\label{lemma_first_kl_bw_conditionals}
 Denoting $\hat{p}_{t_{k-1}|t_k}$ be the conditional probability of
  $\hat{x}_{k-1}$  given $\hat{x}_k$, and let
  ${p}_{t_{k-1}|t_k}$ be the conditional probability of
  $x_{k-1}$ given $x_k$ using two steps of ODE and one step of noise similar to our algorithm. Then, for the updates in Eq. \ref{eq_empirical_gen_first}, Eq. \ref{eq_empirical_gen_second} (the updates of our Algorithm \ref{alg:1} for each interval given the same starting point for both is the true process at $t_k$:$x_{k}$):
  \[
    \KL{{p}_{t_{k-1}|t_k}(\cdot|{x}_{k})}{\hat{p}_{t_{k-1}|t_k}(\cdot|{x}_k)}
    =  e^{-2h_{k-1}}
    \frac{\|
      x_{k-0.5}-\hat{x}_{k-0.5}\|_2^2}
    {2 (1  -e^{-2h_{k-1}})} 
  \] 
where $h_{k}=t_{k}-t_{k-1}$ denotes the step size, $x_{k-0.5}$ corresponds to two steps of Probability Flow ODE from $x_k$ and thus, the law is same as of the forward process at time $t_k-h_{k}-h_{k-1}$.
\end{lemma}
\noindent
\begin{proof}
For this, we know that from Algorithm~\ref{alg:1} that the conditional $\hat{p}_{t_{k-1}|t_k}(\cdot|{x}_{k})$ for the generation process is the following Gaussian: 
\[
\hat{p}_{t_{k-1}|t_k}(\cdot|{x}_{k}) \sim \mathcal{N}\left(e^{-h_{k-1}} \hat{x}_{k-0.5}, \left(1 - e^{-2h_{k-1}}\right) I_d \right)
\]
where $I_d$ is the d-dimensional identity matrix. Similarly, for the true process we can just write:
\[
{p}_{t_{k-1}|t_k}(\cdot|{x}_{k}) \sim \mathcal{N}\left(e^{-h_{k-1}} {x}_{k-0.5}, \left(1 - e^{-2h_{k-1}}\right) I_d \right)
\]
Now, since the covariance matrices are same for both, we can just use the following formulae for calculating KL between two Gaussians with different means but same variance:
\[
\KL{{p}_{t_{k-1}|t_k}(\cdot|{x}_{k})}{\hat{p}_{t_{k-1}|t_k}(\cdot|{x}_k)} = \frac{1}{2}(\mu_1-\mu_2)^\top\Sigma^{-1}(\mu_1-\mu_2)  
\]
where $\mu_1$, $\mu_2$ corresponds to the mean of the two distributions and $\Sigma$ corresponds to their covariance. For this case, we have:
\begin{align*}
\mu_1 &= e^{-h_{k-1}} \hat{x}_{k-0.5} \\
\mu_2 &= e^{-h_{k-1}} {x}_{k-0.5} \\
\Sigma &= \left(1 - e^{-2h_{k-1}} \right) I_d
\end{align*}
Merely substituting these values in the KL formulae will lead to the desired term. \\
\end{proof}

\noindent
For KL-aggregation, we have the following lemma:
\begin{lemma}
\label{lemma_second_KL_data_processing_ineq}
  For the discretization sequence $t_1,...,t_{K+1}$ and the law corresponding to the generation process in Algorithm \ref{alg:1}, we will have (where ${p}_{t_k}$ denotes the law of true process at time $t_k$):
  \begin{align*}
    \KL{{p}_{t_1}}{\hat{p}_{t_1}}
    \leq & \KL{{p}_{t_1, t_2, ..., t_{K}, t_{K+1}}}{\hat{p}_{t_1, t_2, ..., t_{K},t_{K+1}}}\\
    =& \KL{{p}_{t_{K+1}}}{\hat{p}_{t_{K+1}}} + 
   \E_{{p}_{t_1,..,t_{K}, t_{K+1}}} \left[ \sum_{k=2}^{K+1}
       \KL{{p}_{t_{k-1}|t_k}(\cdot|{x}_{k})}{\hat{p}_{t_{k-1}|t_k}(\cdot|{x}_k)}\right]
  \end{align*}
\end{lemma}
\begin{proof}
    The first inequality is just the data processing inequality and second equation is the chain rule for KL.
\end{proof}

\subsection{Analysing $\mathbb{E}[\|x_{k-0.5}- \hat{x}_{k-0.5}\|^2_2]$}

\noindent
 We begin by first decomposing the term into a discretization error component and the error due to using the estimated score instead of true score. 

\begin{lemma}
\label{overall_w_2_error_decomp}
    We have:
    \begin{align*}
    {\mathbb{E}[\|x_{k-0.5}- \hat{x}_{k-0.5}\|^2_2]} &\leq 
    2{\mathbb{E}[\|x_{k-0.5}- \tilde{x}_{k-0.5}\|^2_2]} 
    + {2(e^{h_k+h_{k-1}}-1)^2\mathbb{E}[\|s(t_k,x_{k})- \hat{s}(t_k, x_{k})\|_2^2]} \\
    \end{align*}
    where first term in the RHS corresponds to the discretization error $(T_{dis})$ and the second term is the score estimation error $(T_{est})$.
\end{lemma}
\begin{proof}
We have:
\begin{align*}
    \sqrt{\mathbb{E}[\|x_{k-0.5}- \hat{x}_{k-0.5}\|^2_2]} \leq \underbrace{\sqrt{\mathbb{E}[\|x_{k-0.5}- \tilde{x}_{k-0.5}\|^2_2]}}_{T_{\text{dis}}} + \underbrace{\sqrt{\mathbb{E}[\|\tilde{x}_{k-0.5}- \hat{x}_{k-0.5}\|^2_2]}}_{T_{\text{est}}}
\end{align*}
where $\tilde{x}_{k}$ corresponds to the discretization of the true process for the interval $[t_{k-1},t_k]$ defined in Eq. \ref{discretized_true_first}, Eq. \ref{discretized_true_second}.
Squaring both sides and using $2ab \leq a^2 + b^2$, we will have:
\begin{align*}
    {\mathbb{E}[\|x_{k-0.5}- \hat{x}_{k-0.5}\|^2_2]} \leq {2{\mathbb{E}[\|x_{k-0.5}- \tilde{x}_{k-0.5}\|^2_2]}} + {2{\mathbb{E}[\|\tilde{x}_{k-0.5}- \hat{x}_{k-0.5}\|^2_2]}}
\end{align*}

\noindent
\textbf{Bounding $T_{\text{est}}$.} Now, utilizing the Eq. \ref{eq_empirical_gen_first}, Eq. \ref{discretized_true_first}, we have:
\begin{align*}
    \mathbb{E}[\|\tilde{x}_{k-0.5}- \hat{x}_{k-0.5}\|^2_2] = \mathbb{E}\|(e^{h_k+h_{k-1}}-1) \left(s(t_k,x_{k})- \hat{s}(t_k, x_{k})\right)\|_2^2 &= (e^{h_k+h_{k-1}}-1)^2\mathbb{E}[\|s(t_k,x_{k})- \hat{s}(t_k, x_{k})\|_2^2]
\end{align*}
\end{proof}
\noindent
We discuss the analysis (and eventually bounding it) of the discretization error term $T_{dis}$ in the subsequent subsections.

\subsection{Analysing the discretization error along the ODE path}
\label{rescaling_and_discretizaton_error_analysis}
\noindent

\paragraph{Considering a rescaled process.} 
\noindent
We consider a rescaled version of the original OU process (Eq. \ref{fwd_process}) as $z(t) = e^tx(t)$, leading to:
\begin{align}
\label{fwd_rescaled}
    z(t) = y+ \sqrt{e^{2t}-1} \cdot \eta \, ; \qquad \eta \sim \mathcal{N}(0,I_d) 
\end{align}
where $y$ corresponds to the data distribution: $z_0=x_0=y\sim p_{data}$ with the corresponding forward SDE being:
\begin{align*}
    dz(t) = x(t)de^{t} + e^tdx(t) = x(t)e^tdt + e^{t}\left(-x(t)dt + \sqrt{2}d\vw_t\right) = \sqrt{2}e^td\vw_t
\end{align*}
{We denote the law of this process at time $t$ by $q_t(\cdot)$ where $q_t$  is just a pushforward of $p_t(\cdot)$. Also, we denote the score function of this rescaled process as $s_r(t,\cdot)$ where we will have $s_r(t,z(t)) = e^{-t}s(t,e^{-t}z(t))$}.
 The \textbf{Probability Flow ODE} becomes:
\begin{align}
\label{rescaled_back_ode}
    dz(t) = -e^{2t}s_r(t,z(t))dt
\end{align}
Using the Exponential Integrator discretization for a given interval $[t_{k-1},t_k]$ here also we define $\tilde{z}_{k-0.5}$ for this interval when starting from $z_k$ :
\begin{align}
\label{rescaled_discretized_ode_process}
    \tilde{z}_{k-0.5} = {z_k} + \frac{1}{2}e^{2t_{k-2}}({e^{2(h_k+h_{k-1})}-1})s_r(t_k,z_k) 
\end{align}
where $t_{k-2}=t_{k}-h_k-h_{k-1}$
\noindent
Now, we have the following Lemma for bounding $\mathbb{E}\left[\|z_{k-0.5}-\tilde{z}_{k-0.5}\|_2^2\right]$.

\begin{lemma}
\label{discretization_error_rescaled}
    We have:
    \begin{align*}
        \mathbb{E}\left[\|z_{k-0.5}-\tilde{z}_{k-0.5}\|_2^2\right]  \leq \frac{1}{2}(h_k+h_{k-1})^3  \int_{t_{k-2}}^{t_k}e^{4t}\mathbb{E}\left[\left\|s'_r(t,z(t)) \right\|_2^2\right]dt
    \end{align*}
    where $s'_r(t,z(t))$ can be written using the Taylor Theorem as:
\[
s'_r(t,z(t))= {\frac{\partial s_r(t,z)}{\partial t}}
+ {{\frac{\partial s_r(t,z)}{\partial z}} {\frac{d z(t)}{d t}} \Big|_{z=z(t)}}
\]
\end{lemma}

\begin{proof}
    Since, we have used the Exponential Integrator discretization, the ODEs for the interval $[t_{k-2},t_k]$ corresponding to $z_k, \tilde{z}_k$ (given $\tilde{z}_k=z_k$) are:

\[
dz(t) = - e^{2t}s_r(t,z(t))dt \qquad d\tilde{z}(t) = -e^{2t}s_r(t_k,z_k)dt
\]
Therefore, we have:

\begin{align*}
     \tilde{z}_{k-0.5} - {z}_{k-0.5} = \int^{t_{k-2}}_{t_k}d(\tilde{z}(t) - z(t)) & =  
 \int_{t_{k}}^{t_{k-2}} e^{2t}\left(s_r(t_{k},z_k) - s_r(t,z(t))\right)dt \\
 &= 
 \int_{t_{k}}^{t_{k-2}}  {e^{2t}\left(s_r(t_{k},z_{{k}}) - \left(\underbrace{s_r(t_k,z_{k}) + \int^t_{t_k}s'_r(u,z_u)}_{\text{Taylor's Integral Remainder}}du \right)\right)dt } \\ 
\end{align*}
where $s'_r(t,z(t))$ can be written using the Taylor Theorem as:
\[
s'_r(t,z(t))= {\frac{\partial s_r(t,z)}{\partial t}}
+ {{\frac{\partial s_r(t,z)}{\partial z}} {\frac{d z(t)}{d t}} \Big|_{z=z(t)}}
\]
Using this, we will have:
\begin{align*}
\mathbb{E}[\|\tilde{z}_{k-0.5}-z_{k-0.5}\|_2^2] &= \mathbb{E}\left[\left\|\int_{t_k}^{t_{k-2}}dt\int_{t_k}^te^{2u}s'_r(u,z(u))du \right\|_2^2 \right] \\
&\leq \mathbb{E}\left[(h_k+h_{k-1})\int_{t_{k-2}}^{t_k}dt\left\|\int_{t_k}^te^{2u}s'_r(u,z(u))  du\right\|_2^2\right] \\
&\leq \mathbb{E}\left[(h_k+h_{k-1})\int_{t_{k-2}}^{t_k}(t_k-t)\int_{t}^{t_k}\left\|e^{2u}s'_r(u,z(u))  \right\|_2^2dudt\right] \\
        &= (h_k+h_{k-1})\int_{t_{k-2}}^{t_k}(t_k-t)dt\int_{t}^{t_k}\mathbb{E}\left[\left\|e^{2u}s'_r(u,z(u)) \right\|_2^2\right] du \\
        &= (h_k+h_{k-1})\int_{t_{k-2}}^{t_k}e^{4u}\mathbb{E}\left[\left\|s'_r(u,z(u)) \right\|_2^2\right] du\int_{t_{k-2}}^{u}(t_k-t) dt\\
        &\leq \frac{(h_k+h_{k-1})^3}{2}  \int_{t_{k-2}}^{t_k}e^{4u}\mathbb{E}\left[\left\|s'_r(u,z(u)) \right\|_2^2\right]du 
\end{align*}
\end{proof}

\noindent
We now calculate and bound the spatial gradient since the partial gradient w.r.t. time can be written in terms of the spatial gradient using the Fokker Planck Equation (FPE).

\paragraph{Calculating the Jacobian $\nabla s_r(t,z(t))$ for this rescaled process:} 

\noindent
We can now observe the following for this rescaled process:
\[
q_{t}(z) = \int q_{t}(z|y)p_{data}(y)dy \propto \int  e^{-\frac{\|z-y\|^2}{2(e^{2t}-1)}} p_{data}(y)dy  
\]
which takes us to the following formulation of the score function:
\begin{lemma}
\label{score_calc}
    For the rescaled process in eq. (11), we have:
    \begin{align*}
s_r(t,z) = \nabla \log q_{t}(z) = \frac{\nabla q_{t}(z)}{q_{t}(z)} = \frac{\nabla \int e^{-\frac{\|z-y\|^2}{2(e^{2t}-1)}} p_{data}(y) dy}{\int e^{-\frac{\|z-y\|^2}{2(e^{2t}-1)}} p_{data}(y) dy} &= \frac{ \int \frac{y-z}{(e^{2t}-1)} e^{-\frac{\|z-y\|^2}{2(e^{2t}-1)}} p_{data}(y) dy}{\int   e^{-\frac{\|z-y\|^2}{2(e^{2t}-1)}} p_{data}(y) dy}\\&= \int P(y|z) \frac{y-z}{(e^{2t}-1)} dy
\\&= \mathbb{E}_{y|z}\frac{y-z}{(e^{2t}-1)}
\end{align*}
where $P(y|z)  = \frac{P(y,z)}{\int P(y,z)dy}$ and $P(y,z)=  e^{-\frac{\|z-y\|^2}{2(e^{2t}-1)}} p_{data}(y)$. 
\end{lemma}

\begin{lemma}
\label{spatial_der_score}
    \textbf{Gradient of score.} We have the following expression for the Jacobian of the score ${\nabla s_r(t,z)}$ for the rescaled process $z(t)$:
    \begin{align*}
        {\nabla s_r(t,z)} = \mathrm{Var}_{y|z}\left[\frac{y-z}{e^{2t}-1}\right] - \frac{I_d}{e^{2t}-1}
    \end{align*}
    where $\mathrm{Var}$ denotes the covariance matrix.
\end{lemma}

\begin{proof}
    We begin by calculating ({the gradient of P is w.r.t. second variable $z$ in this Lemma}), 
\begin{align*}
    \nabla P(y|z) =  \frac{\nabla P(y,z) \cdot \int P(y,z)dy - P(y,z)\cdot \int \nabla P(y,z)dy}{\left(\int P (z,y) dy\right)^2}
\end{align*}
From the calculations in last Lemma (\ref{score_calc}), we have:
\begin{align*}
    \nabla P(y,z) = \frac{y-z}{e^{2t}-1}P(y,z)
\end{align*}
and therefore:
\begin{align*}
    \nabla P(y|z) =  \frac{y-z}{e^{2t}-1} \cdot \frac{P(y,z)}{\int P(y,z)dy} - \frac{P(y,z)}{\int P(y,z) dy}\cdot \frac{\int \frac{y-z}{e^{2t}-1}P(y,z)dy}{\int P(y,z)dy} = P(y|z) \left(\frac{y-z}{e^{2t}-1} - \mathbb{E}_{y|z}\left[\frac{y-z}{e^{2t}-1}\right]\right)
\end{align*}
Thus, we can calculate the $\nabla s_r(t,z)$ as follows:
\begin{align*}
    \nabla s_r(t,z) &= \int \frac{y-z}{e^{2t}-1} \nabla P(y|z)^\top dy - \frac{I_d}{e^{2t}-1}\int P(y|z)  dy \\
    &= \int P(y|z) \frac{y-z}{e^{2t}-1}\left(\frac{y-z}{e^{2t}-1} - \mathbb{E}_{y|z}\left[\frac{y-z}{e^{2t}-1}\right]\right)^\top dy - \frac{I_d}{e^{2t}-1}\\
    &= \mathbb{E}_{y|z} \left[\frac{y-z}{e^{2t}-1}\left(\frac{y-z}{e^{2t}-1} - \mathbb{E}_{y|z}\left[\frac{y-z}{e^{2t}-1}\right]\right)^\top \right] - \frac{I_d}{e^{2t}-1}\\
    &= \text{Var}_{y|z}\left[\frac{y-z}{e^{2t}-1}\right] - \frac{I_d}{e^{2t}-1}
\end{align*}
\end{proof}

\paragraph{Bounding $\mathbb{E}_{q_t}\left[\|s_r(t,z)\|^p\right]$, $\mathbb{E}_{q_t}\left[\|\nabla s_r(t,z)\|^2\right]$ and other spatial gradient terms.} Since, we know that $\frac{y-z(t)}{\sqrt{e^{2t}-1}}=\epsilon \sim \mathcal{N}(0,I_d)$, we first provide a helper lemma to bound the moment of the multivariate Gaussian distribution. Then using that and the formulae for score function, Jacobian provided in Lemma \ref{score_calc}, we bound $\mathbb{E}\left[\|s_r(t,z)\|^p\right]$, $\mathbb{E}\left[\|\nabla s_r(t,z)\|_F^2\right]$ for a general $p$.

\begin{lemma}
\label{Gaussian_integ}
    \textbf{Gaussian Moment.} We have the following result for the Gaussian random variable $\eta \sim \mathcal{N}(0,I_d)$:
    \begin{align*}
        \mathbb{E}\left[\|\eta \eta^\top\|^p_F\right] =\mathbb{E}\left[\left\|\eta\right\|^{2p}_2\right]\leq (d+2p)^p
    \end{align*}
\end{lemma}
\begin{proof}
    We will have:
    \begin{align*}
        \|\eta \eta^\top\|^2_F = Tr\left((\eta\eta^\top)^\top(\eta\eta^\top)\right) = Tr\left(\eta\eta^\top\eta \eta^\top\right) = \eta^\top\eta Tr(\eta\eta^\top) = (\eta^\top\eta)^2
    \end{align*}
    Thus, we have $\|\eta \eta^\top\|^p_F = (\eta^\top\eta)^p = \|\eta\|^{2p}_2$. Since $\eta \sim \mathcal{N}(0,I_d)$ and thus, the vector $\eta$ has \textit{i.i.d.} normal entries, thereby:
    \begin{align*}
        \|\eta\|^2_2 = \sum_i \eta^2_i \sim \chi^2(d) \quad  \implies  \mathbb{E}[\|\eta\|^{2p}_2] = \mathbb{E}[ (X)^p] \quad \text{where} \, X\sim \chi^2(d)
    \end{align*}
    where $\chi^2(d)$ denotes the chi-squared distribution with $d$ degrees of freedom. Now, we can just use the formulae for moments of $\chi^2(d)$, leading us to:
    \begin{align*}
        \mathbb{E}\left[\|\eta\|^{2p}\right] = \mathbb{E}[X^p] = 2^{p}\cdot\frac{\Gamma (p+\frac{d}{2})}{\Gamma(\frac{d}{2})} \overset{(a) }{ \leq} 2^p \left(\frac{d}{2}+p\right)^p = (d+2p)^p
    \end{align*}
    where  for inequality $(a)$, we have just used the gamma function bound.
\end{proof}

\begin{lemma}
\label{score_bound}
    We have:
    \begin{align*}
        \mathbb{E}_{q_t}\left[\|s_r(t,z(t))\|^2\right] \leq \frac{d}{e^{2t}-1} \, ; \quad  \mathbb{E}_{q_t}\left[\|s_r(t,z(t))\|^p\right]\leq \frac{(d+p)^{p/2}}{(e^{2t}-1)^{p/2}} \, ; \quad \mathbb{E}_{q_t}\left[\|\nabla s_r(t,z)\|^2_F\right] \leq \frac{2d^2+6d}{(e^{2t}-1)^2}
    \end{align*}
\end{lemma}
\begin{proof}
    Similar to the \cite{chen2023improved} paper, here we also utilize the fact that $\frac{y-z(t)}{\sqrt{1-e^{-2t}}}$ is Gaussian, using Lemma \ref{score_calc} we have:
    \begin{align*}
        \mathbb{E}_{q_t}\left[\|s_r(t,z)\|^2\right] = \frac{1}{(e^{2t}-1)}\mathbb{E}_{z\sim q_t}\left[\left\|\mathbb{E}_{y|z}\left[\frac{y-z}{\sqrt{e^{2t}-1}}\right]\right\|^2\right] &\leq \frac{1}{(e^{2t}-1)}\mathbb{E}_{q_t}\mathbb{E}_{y|z}\left[\left\|\left[\frac{y-z}{\sqrt{e^{2t}-1}}\right]\right\|^2\right] \\
        &= \frac{1}{(e^{2t}-1)}\mathbb{E}_{\eta \sim \mathcal{N}(0,I_d)}\left[\|\eta\|^2\right]\\
        &= \frac{d}{e^{2t}-1}\\
    \end{align*}
    \noindent
    For a general $p\geq 2$, it becomes:
     \begin{align*}
        \mathbb{E}_{q_t}\left[\|s_r(t,z(t))\|^p\right] = \frac{1}{(e^{2t}-1)^{p/2}}\mathbb{E}_{q_t}\left[\left\|\mathbb{E}_{y|z(t)}\left[\frac{y-z}{\sqrt{(e^{2t}-1)^{p/2}}}\right]\right\|^p\right] &\leq \frac{1}{(e^{2t}-1)^{p/2}}\mathbb{E}_{q_t}\mathbb{E}_{y|z(t)}\left[\left\|\left[\frac{y-z}{\sqrt{e^{2t}-1}}\right]\right\|^p\right] \\
        &= \frac{1}{(e^{2t}-1)^{p/2}}\mathbb{E}_{\eta \sim \mathcal{N}(0,I_d)}\left[\|\eta\|^p\right]\\
        &\leq \frac{(d+p)^{p/2}}{(e^{2t}-1)^{p/2}} \tag{Lemma \ref{Gaussian_integ}}\\
    \end{align*}
    Going similarly, a \textit{naive} bound on $\mathbb{E}_{q_t}\left[\|\nabla s_r(t,z)\|^2_F\right]$ will be:
    \begin{align*}
        \mathbb{E}_{z\sim q_t}\left[\|\nabla s_r(t,z)\|^2_F\right] &= \mathbb{E}_{z\sim q_t}\left[\left\|\text{Var}_{y|z}\left[\frac{y-z}{e^{2t}-1}\right] - \frac{I_d}{e^{2t}-1}\right\|^2_F\right] \tag{Lemma \ref{spatial_der_score}} \\
        &\leq 2\mathbb{E}_{z\sim q_t}\left[\left\|\mathbb{E}_{y|z}\left[\frac{y-z}{e^{2t}-1}\right]\left[\frac{y-z}{e^{2t}-1}\right]^{\top}\right\|^2_F\right] + \frac{2d}{(e^{2t}-1)^2}\\
        & \leq 2\mathbb{E}_{z\sim p_t}\mathbb{E}_{y|z}\left[\left\|\left[\frac{y-z}{e^{2t}-1}\right]\left[\frac{y-z}{e^{2t}-1}\right]^{\top}\right\|^2_F\right] + \frac{2d}{(e^{2t}-1)^2}\\
        &=2\frac{1}{(e^{2t}-1)^2}\mathbb{E}_{\eta \sim \mathcal{N}(0,I_d)}\left[\|\eta\|^4\right] + \frac{2d}{(e^{2t}-1)^2} \\
        &= \frac{2d^2+6d}{(e^{2t}-1)^2}
    \end{align*}
\end{proof}

\subsubsection{Expressing the $\mathbb{E}_{p_t}\left[\|s'(t,x(t))\|^2\right]$ in terms of spatial derivatives}

\paragraph{Fokker Plank Equation (FPE) to relate time and spatial derivative.} Since we also need to bound $\partial_t s_r(t,z)$ to bound the $s_r'(t,z)$, we will utilise the Fokker Plank equation associated with the forward/reverse processes which relates the partial derivative w.r.t. t with the spatial derivative. The \textbf{Fokker Plank Equation} corresponding to the rescaled process $z(t)$ (Eq.\ref{fwd_rescaled}) would be:
\begin{align}
\label{fpe_rescaled}
    \partial_tq_t(z) = -\sum_{i=1}^d \partial_i \left(-e^{2t}\nabla q_t(z)\right) = e^{2t}\Delta q_t(z)
\end{align}

\noindent
Since score function is just $\frac{\nabla q_t(z(t))}{q_t}$, we provide the corresponding \textit{score-fpe} for the rescaled process to relate the $\partial_t s_r(t,z)$ with spatial derivative in the Lemma below:

\begin{lemma}
\label{score_fpe}
 We have the following Fokker Plank Equation corresponding to the score function for this rescaled process:
    \begin{align}
    \label{score_fpe_eq}
    \partial_t s_r(t,z) = e^{2t} \Delta s_r(t,z) + e^{2t}\nabla \|s_r(t,z)\|^2 = e^{2t} \Delta s_r(t,z) + 2e^{2t}\nabla s_r(t,z)^\top s_r(t,z)
\end{align}
\end{lemma}

\begin{proof}
    \noindent
To arrive at the equation for the score function, we first derive an equation for $\partial_t \log q_t$ by considering the following term:
\begin{align*}
     e^{2t}\sum_{i=1}^d \partial_i \left(\nabla \log q_t(z)\right) = e^{2t}\sum_{i=1}^d \partial_i \left(\frac{\nabla q_t(z)}{q_t(z)} \right) &= e^{2t}\sum_{i=1}^d  \left(\frac{q_t(z)\partial_i\nabla q_t(z)- \partial_i q_t(z)\nabla q_t(z)}{p^2_t(z)} \right) \\
     &= e^{2t}\sum_{i=1}^d  \left(\frac{\partial_i\nabla q_t(z)}{q_t(z)} \right) - e^{2t}\|\nabla\log q_t(z)\|^2  \\
\end{align*}
which results in:
\begin{align*}
  \partial_t \log q_t(z) =  \frac{\partial_tq_t(z)}{q_t(z)} = e^{2t}\sum_{i=1}^d  \left(\frac{\partial_i\nabla q_t(z)}{q_t(z)} \right) = e^{2t}\sum_{i=1}^d \partial_i \left(\nabla \log q_t(z)\right) + e^{2t} \|\nabla\log q_t(z)\|^2
\end{align*}
Now, again taking a spatial gradient:
\begin{align*}
    \nabla\partial_t \log q_t(z) = e^{2t}\sum_{i=1}^d \nabla\partial_i \left(\nabla \log q_t(z)\right) + e^{2t}\|\nabla\log q_t(z)\|^2
\end{align*}
Interchanging the operators result in \textit{\textbf{score fpe} for the forward path} (on the reverse it would be negative):
\begin{align*}
    \partial_t s_r(t,z) = e^{2t} \Delta s_r(t,z) + e^{2t}\nabla \|s_r(t,z)\|^2 = e^{2t} \Delta s_r(t,z) + 2e^{2t}\nabla s_r(t,z)^\top s_r(t,z)
\end{align*}
\end{proof}

\noindent
Now  based on this lemma, we express  the overall derivative in terms of spatial derivative in the following lemma. 
\begin{lemma}
\label{overall_to_spatial}
    We have the following relation for the overall score derivative  $s_r'(t,z)$ and $\nabla s_r(t,z)$ for the rescaled process in Eq. \ref{fwd_rescaled}:
    \begin{align}
    \label{eq_overall_to_spatial}
        &\mathbb{E}_{q_t}[\|s_r'(t,z)\|^2] = {\mathbb{E}_{q_t}\left[{ e^{4t}\| \Delta s_r(t,z) \|^2_2} + e^{4t}{\left\|\nabla s_r(t,z)^\top s_r(t,z)\right\|_2^2} + 2e^{4t}(\Delta s_r(t,z))^\top\left(\nabla s_r(t,z)^\top s_r(t,z)\right)\right]}
    \end{align}
    where (recall from Section \ref{notations}) $\Delta$ denotes the Laplacian of a vector.
\end{lemma}
\begin{proof}
    Now, utilising the Fokker Planck Equation (FPE) for the score function, we will have:
    \begin{align*}
        &\mathbb{E}_{q_t}[\|s_r'(t,z(t))\|^2] \\
        &= \mathbb{E}_{q_t}[s_r'(t,z)^\top s_r'(t,z)] \\ 
    &= \mathbb{E}_{q_t}\left[\left(\partial_t s_r(t,z) + \nabla s_r(t,z)^\top\left(\frac{d z}{dt}\right)\right)^\top\left(\partial_t s_r(t,z) + \nabla s_r(t,z)^\top\left(\frac{d z}{dt}\right)\right)\right] \\
    &= \mathbb{E}_{q_t}\left[\left(\partial_t s_r(t,z) - e^{2t}\nabla s_r(t,z)^\top s_r(t,z)\right)^\top\left(\partial_t s_r(t,z) - e^{2t}\nabla s_r(t,z)^\top s_r(t,z)\right)\right] \qquad \tag{reverse ODE Eq. \ref{rescaled_back_ode}}\\
     &= {\mathbb{E}_{q_t}\left[{ \|\partial_t s_r(t,z)\|^2_2} + e^{4t}{\left\|\nabla s_r(t,z)^\top s_r(t,z)\right\|_2^2} - 2e^{2t}{\partial_t s_r(t,z)^\top \left(\nabla s_r(t,z)^\top s_r(t,z)\right)}\right]} \\
     &= {\mathbb{E}_{q_t}\Bigg[{ \|e^{2t} \Delta s_r(t,z) + 2e^{2t}s_r(t,z)^\top\nabla s_r(t,z)\|^2_2} + e^{4t}{\left\|\nabla s_r(t,z)^\top s_r(t,z)\right\|_2^2}} \\
     &\qquad \qquad {-2e^{2t}{(e^{2t} \Delta s_r(t,z) + 2e^{2t}s_r(t,z)^\top\nabla s_r(t,z))^\top \left(\nabla s_r(t,z)^\top s_r(t,z)\right)}\Bigg] \tag{FPE for score, Lemma \ref{score_fpe}}} \\
        &= {\mathbb{E}_{q_t}\left[{ e^{4t}\| \Delta s_r(t,z) \|^2_2} + e^{4t}{\left\|\nabla s_r(t,z)^\top s_r(t,z)\right\|_2^2} + 2e^{4t}(\Delta s_r(t,z))^\top\left(\nabla s_r(t,z)^\top s_r(t,z)\right)\right]} \\
    \end{align*}
\end{proof}

\subsection{Bounding the required spatial derivative terms}

\noindent
Now to bound the terms $\mathbb{E}_{q_t}\left[\|\Delta s_r(t,z) \|^2_2\right]$, $\mathbb{E}\left[\|s_r(t,z)\nabla s_r(t,z) \|^2_2\right]$ appearing in Eq. \ref{eq_overall_to_spatial}, we analyze the relationship of these spatial gradient terms with $\frac{d}{dt}\mathbb{E}_{q_t}\left[\|s_r(t,z)\|^2\right]$. Since the naive bound on $\mathbb{E}_{q_t}\left[\|s_r(t,z)\|^2\right]$ is proportional to $d$ as against $d^2$ in case of $\mathbb{E}_{q_t}\left[\|\nabla s_r(t,z)\|_F^2\right]$, this can lead to improved $d-$dependence of the discretization error upon integrating this for a given time interval. We first discuss two lemmas: the first one establishes relation beween $\frac{d}{dt}\mathbb{E}_{q_t}[\|s_r(t,z)\|^2]$ and 
$\mathbb{E}_{q_t}[\|\nabla s_r(t,z)\|^2_F]$. Then extending this lemma for general power $m$ in $\frac{d}{dt}\mathbb{E}_{q_t}[\|s_r(t,z)\|^m]$ leads to the terms comprising both $\nabla s_r$ and $s_r$  from which we can bound the term $\mathbb{E}\left[\|s_r(t,z)\nabla s_r(t,z) \|^2_2\right]$. Then, utilising these lemmas and bounding terms comprising $\delta q_t$ and $\nabla s_r$ in terms of $\frac{d}{dt}\mathbb{E}_{q_t}\left[\|s_r(t,z)\|^2\right]$, we finally bound $\mathbb{E}_{q_t}\left[\|\Delta s_r(t,z) \|^2_2\right]$ by applying Integration By Parts.

\subsubsection{Establishing relation between score and its first order spatial gradient terms}
\noindent
We first analyze the term $\frac{d}{dt}\mathbb{E}_{q_t}[\|s_r(t,z)\|^2]$ and manipulate it to relate it with $\mathbb{E}_{q_t}[\|\nabla s_r(t,z)\|^2_F]$ leading to the following lemma.

\begin{lemma}
\label{spatial_time_relation}
 We have:
\begin{align*}
    \frac{d}{dt}\mathbb{E}_{q_t}[\|s_r(t,z)\|^2] = -2e^{2t}\mathbb{E}_{q_t}[\|\nabla s_r(t,z)\|^2_F]
\end{align*}
\end{lemma}

\begin{proof}
We begin by analysing the LHS term, taking the derivative inside the integral and utilise Fokker Plank Equations (FPE) (for $q_t$ and $s_r$ (Eq. \ref{fpe_rescaled}, Eq. \ref{score_fpe_eq}) for the rescaled process to convert it to spatial derivative and finally utilise Integration By Parts (IBP):
\begin{align*}
    &\frac{d}{dt}\mathbb{E}_{q_t}[\|s_r(t,z)\|^2]  \\
    &= \frac{d}{dt} \int q_t(z)\|s_r(t,z)\|^2 dz\\
     &= \int \partial _t q_t(z)\|s_r(t,z)\|^2dz +  2\int q_t(z)s_r(t,z)^\top\partial_t s_r(t,z)dz\\
      &= \int e^{2t}\Delta q_t(z)\|s_r(t,z)\|^2dz +  2\int q_t(z)s_r(t,z)^\top\left(e^{2t} \Delta s_r(t,z) + e^{2t}\nabla \|s_r(t,z)\|^2\right)dz \tag{FPE Eq. \ref{fpe_rescaled}, \ref{score_fpe_eq}}\\
      &= -e^{2t}\int \nabla q_t(z) \cdot \nabla \|s_r(t,z)\|^2dz +  2e^{2t}\int q_t(z)s_r(t,z)^\top\left(\Delta s_r(t,z) + \nabla \|s_r(t,z)\|^2\right)dz \qquad \tag{IBP}\\
      &= e^{2t}\int \nabla q_t(z) \cdot \nabla \|s_r(t,z)\|^2dz +  2e^{2t}\int q_t(z)s_r(t,z)^\top\left(\Delta s_r(t,z)\right)dz \qquad \qquad \text{( $q_t(z)s_r(t,z) =\nabla q_t(z)$)}\\
      &= e^{2t}\int \nabla q_t(z) \cdot \nabla \|s_r(t,z)\|^2dz +  2e^{2t}\int \underbrace{q_t(z)\Delta s_r(t,z)^\top}_{\text{jointly for IBP}}s_r(t,z)dz \\
      &= e^{2t}\int \nabla q_t(z) \cdot \nabla \|s_r(t,z)\|^2dz -  2e^{2t}\int {\nabla q_t(z)\cdot \nabla s_r(t,z)^\top}s_r(t,z)dz - 2e^{2t}\int  q_t(z) \|\nabla s_r(t,z)\|^2_Fdz \tag{IBP} \\
      &=- 2e^{2t}\int  q_t(z) \|\nabla s_r(t,z)\|^2_Fdz  \tag{$\nabla\|s_r(t,z)\|^2 = 2(\nabla s_r(t,z))^\top s_r(t,z)$}\\
\end{align*}
\end{proof}

\noindent
We now generalize this lemma by establishig relation between $\frac{d}{dt}\mathbb{E}_{q_t}[\|s_r(t,z)\|^m]$ for a general $m\geq 2 $ and $\mathbb{E}_{q_t}[\|\nabla s_r(t,z)\|^2_F]$. For $m>2$, the RHS should have terms comprising both $\nabla s_r(t,z), s_r(t,z)$ and thus the result can be used to bound the second part of the RHS in Eq. \ref{eq_overall_to_spatial}.

\begin{lemma}
\label{generalized_lemma}
We have the following general result for any $m\geq 2$:
\begin{align*}
    e^{-2t}\frac{d}{dt}\mathbb{E}_{q_t}[\| s_r(t,z)\|_2^m] = -m\mathbb{E}_{q_t}[\|s_r(t,z)\|^{m-2}_2\|\nabla s_r(t,z)\|^2_F] -\frac{m(m-2)}{4} \mathbb{E}_{q_t}\left[\|s_r(t,z)\|^{m-4}_2\left\| \left(\nabla \|s_r(t,z)\|^2_2\right)\right\|^2_2\right]
\end{align*}
\end{lemma}

\begin{proof}
Here also, we start with analyzing the LHS similar to previous lemma (as discussed in Section \ref{notations}, $\partial_i$ corresponds to partial derivative w.r.t. $i^{th}$ coordinate of $z$, $\Delta=\sum_{i}\partial_i\partial_i$ is the Laplacian, $s_r(\cdot)_i$ corresponds to $i^{th}$ element of $s_r$ which implies $\|s_r(t,z)\|^2 = \sum_{i=1}^d s^2_r(t,z)_i$) (all the variables under $\sum$ range from $1$ to $d$ if not mentioned):
\begin{align*}
   & e^{-2t}\frac{d}{dt}\mathbb{E}_{q_t}[\| s_r(t,z)\|_2^m]\\
    &\overset{(a)}{=} \int e^{-2t} \partial_t q_t(z) \|s_r(t,z)\|^m_2 dz + m\int e^{-2t} q_t(z)  \|s_r(t,z)\|^{m-2}_2 s_r(t,z)^\top\partial_t s_r
    (t,z) dz  \\
    &\overset{(b)}{=} \sum_{i=1}^d\int  \partial_i \partial_i q_t(z) \|s_r(t,z)\|^m_2 dz + m\int  q_t(z)  \|s_r(t,z)\|^{m-2}_2 \left( \sum^d_{i,j=1} s_{r}(t,z)_j \left(\partial_i \partial_i s_{r}(t,z)_j +  \partial_i s^2_{r}(t,z)_j\right)\right) dz\\
    &\overset{(c)}{=} -\frac{m}{2}\sum^d_{i,j=1}\int  \partial_i q_t(z) \|s_r(t,z)\|^{m-2}_2 \partial_i s_r^2(t,z)_jdz \\
    & \qquad \qquad \qquad + m \sum^d_{i,j=1}\int  q_t(z)  \|s_r(t,z)\|^{m-2}_2 s_r(t,z)_j \left(\partial_i \partial_i s_r(t,z)_j +  \partial_i s^2_r(t,z)_j\right) dz\\
     &\overset{(d)}{=} m\sum_{i,j}\int  q_t(z) \|s_r(t,z)\|^{m-2}
    s_r(t,z)_i s_r(t,z)_i\partial_i s_r(t,z)_jdz + m \sum_{i,j}\int  \underbrace{q_t(z)  \|s_r(t,z)\|^{m-2}_2 s_r(t,z)_j}_{I_1} \partial_i \partial_i s_r(t,z)_j dz\\
      &\overset{(e)}{=} m\sum_{i,j}\int  q_t(z) \|s_r(t,z)\|^{m-2}
    s_r(t,z)_i s_r(t,z)_i\partial_i s_r(t,z)_jdz - m \sum_{i,j}\int  { \partial_i q_t(z)  \cdot \|s_r(t,z)\|^{m-2}_2 s_r(t,z)_j}  \partial_i s_r(t,z)_j  dz\\
      & \quad  -  m \sum_{i,j}\int  {  q_t(z)  \cdot \partial_i \|s_r(t,z)\|^{m-2}_2 \cdot s_r(t,z)_j}  \partial_i s_r(t,z)_j  dz -  m \sum_{i,j}\int  {  p_t(z)  \cdot \|s_r(t,z)\|^{m-2}_2 \cdot \partial_i s_r(t,z)_j}  \partial_i s_r(t,z)_j dz \\
      &\overset{(f)}{=} -  m(m-2) \sum_{i,j,k}\int  {  p_t(z)  \cdot \|s_r(t,z)\|^{m-4}_2 \cdot s_r(t,z)_k}  \partial_i s_r(t,z)_k s_r(t,z)_j\cdot \partial_i s_r(t,z)_j    dz \\
    & \qquad \qquad \qquad -  m \sum_{i,j}\int    q_t(z)  \cdot \|s_r(t,z)\|^{m-2}_2 \cdot   \left(\partial_i s_r(t,z)_j\right)^2 dz \\
    & = -\frac{m(m-2)}{4} \mathbb{E}_{q_t}\left[\|s_r(t,z)\|^{m-4}_2\left\| \left(\nabla \|s_r(t,z)\|^2_2\right)\right\|^2_2\right]-m\mathbb{E}_{q_t}\left[\|s_r(t,z)\|^{m-2}_2\|\nabla s_r(t,z)\|^2_F\right] 
\end{align*}
where in $(a)$ we have used $\partial_t\|s_r(t,z)\|^{m}_2=m\|s_r(t,z)\|^{m-2}_2 s_r(t,z)^\top\partial_t s_r(t,z)$, $(b)$  implies the use of FPEs Eq. \ref{fpe_rescaled}, Lemma \ref{score_fpe} , $(c)$ is the application of Integration By Parts on the first term, $(d)$ uses $\partial_i q_t(z) = q_t(z)s_r(t,z)_i$ then subtract it from the second part of the second term and use $\partial_is^2_r(t,z)_j=2s_r(t,z)_j\partial_is_r(t,z)_j$, $(e)$ implies again using Integration By Parts on the second term where one term is jointly considered as $I_1$ and the other remaining. $(f)$ is derived using $\partial _t q = q_t(z)s_r(t,z)$ on second term, cancelling the first two terms and writing $\partial_i \|s_r(t,z)\|^{m-2}_2 = (m-2)\|s_r(t,z)\|^{m-4}_2\sum^d_{k=1}s_r(t,z)_k  \partial_i s_r(t,z)_k $.
\end{proof}

\noindent
Now, as discussed before, a consequence lemma of this lemma is that we can bound the second term comprising both $s_r$ and $\nabla s_r$ in our main Eq. \ref{eq_overall_to_spatial} (since the RHS contains these terms for a general m). This is stated as a lemma below.

\begin{lemma}
\label{gen_eq_rhs_terms_bound}
    Defining $X_m = \int q_t(z)\|s_r(t,z)\|^{m-2}\|\nabla s_r(t,z)\|^2_F dz$, $X'_m = \int q_t(z)\|s_r(t,z)\|^{m-4}\|\nabla \|s_r(t,z)\|^2_2\|^2_2 dz$ we can bound it as follows for any $m\geq 2$:
    \begin{align*}
        X_m \leq -\frac{1}{m}e^{-2t}\frac{d}{dt}\mathbb{E}_{q_t}[\|s_r(t,z)\|^m], \quad X'_m \leq -\frac{4}{m(m-2)}e^{-2t}\frac{d}{dt}\mathbb{E}_{q_t}[\|s_r(t,z)\|_2^m]
    \end{align*}
\end{lemma}
\begin{proof}
    For this, considering the Lemma \ref{generalized_lemma}, we have:
    \begin{align*}
        \frac{-1}{m}e^{-2t}\frac{d}{dt}\mathbb{E}_{q_t}[\| s_r(t,z)\|_2^m] = \underbrace{\mathbb{E}_{q_t}[\|s_r(t,z)\|^{m-2}_2\|\nabla s_r(t,z)\|^2_F]}_{X_m} + \frac{(m-2)}{4} \underbrace{\mathbb{E}_{q_t}\left[\|s_r(t,z)\|^{m-4}_2\left\| \left(\nabla \|s_r(t,z)\|^2_2\right)\right\|^2_2\right]}_{X'_m}
    \end{align*}
    When $m\geq 2$, we will have $X_m,X'_m\geq 0$, thus, $X_m \leq -\frac{1}{m}e^{-2t}\frac{d}{dt}\mathbb{E}_{q_t}[\|s_r(t,z)\|^m]$ and \\ $X'_m  \leq -\frac{4}{m(m-2)}e^{-2t}\frac{d}{dt}\mathbb{E}_{q_t}[\|s_r(t,z)\|^m]$.
\end{proof}

\subsubsection{Bounding the second order spatial gradient of score term}
\noindent
For this, we first provide two lemmas to bound terms comprising second order spatial derivative of the law $q_t$ and first order spatial derivative of the score $s_r$.
\begin{lemma}
\label{nabla_p_by_p_calc}
    For the rescaled process $z(t)$, we have:
    \begin{align*}
        \frac{\Delta q_t(z)}{q_t(z)} = \frac{-d}{(e^{2t}-1)} + \mathbb{E}_{y|z}\left[\frac{\|y-z\|^2_2}{(e^{2t}-1)^2}\right]
    \end{align*}
\end{lemma}
\begin{proof}
Here also similar to the score function calculation in Lemma \ref{score_calc}, we have:
    \begin{align*}
        \frac{\Delta q_t(z)}{q_t(z)} = \frac{\nabla \cdot \nabla q_t(z)}{q_t(z)} &\overset{(a)}{=} \frac{ \nabla \cdot \int \frac{y-z}{(e^{2t}-1)} e^{-\frac{\|z-y\|^2}{2(e^{2t}-1)}} p_{data}(y) dy}{\int   e^{-\frac{\|z-y\|^2}{2(e^{2t}-1)}} p_{data}(y) dy} \\
        &= \frac{   \int \left(\nabla\cdot \frac{y-z}{(e^{2t}-1)}\right) e^{-\frac{\|z-y\|^2}{2(e^{2t}-1)}} p_{data}(y) dy + \int  \left(\nabla e^{-\frac{\|z-y\|^2}{2(e^{2t}-1)}} p_{data}(y) \right)  \frac{y-z}{(e^{2t}-1)} dy }{\int   e^{-\frac{\|z-y\|^2}{2(e^{2t}-1)}} p_{data}(y) dy} \\
         &= \frac{   \int \frac{-d}{(e^{2t}-1)} e^{-\frac{\|z-y\|^2}{2(e^{2t}-1)}} p_{data}(y) dy + \int e^{-\frac{\|z-y\|^2}{2(e^{2t}-1)}} p_{data}(y) \left(\frac{y-z}{(e^{2t}-1)}\right)^\top\frac{y-z}{(e^{2t}-1)}     dy }{\int   e^{-\frac{\|z-y\|^2}{2(e^{2t}-1)}} p_{data}(y) dy} \\
         &=  \frac{-d}{(e^{2t}-1)} + { \int P(y|z) \left(\frac{y-z}{(e^{2t}-1)}\right)^\top\frac{y-z}{(e^{2t}-1)}     dy } \\
         &=  \frac{-d}{(e^{2t}-1)} + \mathbb{E}_{y|z}\left[\frac{\|y-z\|^2_2}{(e^{2t}-1)^2}\right] \\
    \end{align*}
    where in $(a)$ we have taken the expression also used in Lemma \ref{score_calc} and as discussed before $y\sim p_{data}$. 
\end{proof}

\begin{lemma}
\label{hess_p_grad_s_norm}
    We have the following result for the rescaled process $z(t)$:
    \begin{align*}
         \int \Delta q_t(z) \|\nabla s_r(t,z)\|^2_F dz \leq  \frac{C_dd^2}{(e^{2t}-1)^{3}}    -  \frac{e^{-2t}de}{2(1+\frac{1}{\log d})(e^{2t}-1)}\frac{d}{dt}\mathbb{E}_{q_t}[\|s_r(t,z)\|^2]  
    \end{align*}
    where $C_d =\frac{(1+2\frac{\log d}{d}+\frac{6}{d})^{\log d+3}}{(1+\log d)}$.
\end{lemma}

\begin{proof}
    We start by writing Laplacian as $\sum_{i}\partial_i\partial_i$ and decomposing the term as follows:
    \begin{align*}
     \int  \Delta q_t(z) \|\nabla s_r(t,z)\|^2_F dz &= \int  \sum_i \partial_i\partial_i q_t(z) \|\nabla s_r(t,z)\|^2_F dz \\
      &= \int  \sum_i \partial_i\partial_i q_t(z) \|\nabla s_r(t,z)\|^2_F dz \\
      &= \int   q_t(z) J_t(z) \|\nabla s_r(t,z)\|^2_F dz  \\
      &= \int   q_t(z) J_t(z) c_m^{-1}\|\nabla s_r(t,z)\|^{2/l}_F \cdot c_m \|\nabla s_r(t,z)\|^{2/m}_F  dz \\
      &\overset{(b)}{\leq}  \frac{c_m^{-l}}{l}\int   q_t(z) J^l_t(z)\|\nabla s_r(t,z)\|^{2}_F dz + \frac{c^m_m}{m}\int q_t(z) \|\nabla s_r(t,z)\|^{2}_F  dz 
    \end{align*}
    where $J_t(z) = \frac{\sum_i\partial_i\partial_i q_t(z)}{q_t(z)}$, $l$ and $m$ are constants where $l>1$, $1/l+1/m=1$ and in $(b)$, we have just used $ab \leq \frac{1}{l}a^l + \frac{1}{m}b^m$ with $a = J_t(z)c^{-1}_m, \, b=c_m$. Now, we utilise Lemma \ref{spatial_der_score} for the spatial gradient of score, Lemma \ref{nabla_p_by_p_calc} to write $J_t(z) = -\frac{d}{e^{2t}-1} + \mathbb{E}_{y|z}\left[ \frac{\|y-z\|^2_2}{(e^{2t}-1)^2}\right]$. Also, for the second term, we can just use Lemma \ref{spatial_time_relation}, leading to :
    \begin{align*}
       &=\frac{c_m^{-l}}{l}\int   q_t(z) \left(\frac{-d}{e^{2t}-1} + \mathbb{E}_{y|z} \left[\frac{\|y-z\|_2^2}{(e^{2t}-1)^2}\right]\right)^l\left\|\text{Var}_{y|z}\left[\frac{y-z}{e^{2t}-1}\right] - \frac{I_d}{e^{2t}-1}\right\|^2_F dz -\frac{e^{-2t}}{2} \frac{c^m_m}{m}\frac{d}{dt}\mathbb{E}_{q_t}[\|s_r(t,z)\|^2] \\
       &\leq\frac{c_m^{-l}}{l}\mathbb{E}_{q_t} \left[\left( \mathbb{E}_{y|z} \left[\frac{\|y-z\|_2^2}{(e^{2t}-1)^2}\right]\right)^l\left(\left(\mathbb{E}_{y|z}\left[\frac{\|y-z\|^2_2}{(e^{2t}-1)^2}\right]\right)^2 +\frac{d}{(e^{2t}-1)^2}\right) \right]  -\frac{e^{-2t}}{2} \frac{c^m_m}{m}\frac{d}{dt}\mathbb{E}_{q_t}[\|s_r(t,z)\|^2] \\
       &\leq\frac{c_m^{-l}}{l}\mathbb{E}_{q_t} \mathbb{E}_{y|z}\left[\left(  \left[\frac{\|y-z\|_2^2}{(e^{2t}-1)^2}\right]\right)^{l+2} \right] + \frac{d}{(e^{2t}-1)^2}\frac{c_m^{-l}}{l}\mathbb{E}_{q_t} \mathbb{E}_{y|z}\left[\left(  \left[\frac{\|y-z\|_2^2}{(e^{2t}-1)^2}\right]\right)^{l} \right]  -\frac{e^{-2t}}{2} \frac{c^m_m}{m}\frac{d}{dt}\mathbb{E}_{q_t}[\|s_r(t,z)\|^2]   \tag{Jensen's Inequality}\\
       &=\frac{c_m^{-l}}{l(e^{2t}-1)^{l+2}}\mathbb{E}_{\eta\sim\mathcal{N}(0,I_d)} \left[  {\|\eta\|_2}^{2l+4} + d\|\eta\|^{2l}_2 \right]   {-\frac{e^{-2t}}{2} \frac{c^m_m}{m}\frac{d}{dt}\mathbb{E}_{q_t}[\|s_r(t,z)\|^2]}    \qquad \qquad \tag{since $\frac{y-z}{\sqrt{e^{2t}-1}}\sim \mathcal{N}(0,I_d)$ } \\
       &{\leq  \frac{c_m^{-l}}{l(e^{2t}-1)^{l+2}} \cdot (d+2l+4)^{l+2}   -\frac{e^{-2t}}{2} \frac{c^m_m}{m}\frac{d}{dt}\mathbb{E}_{q_t}[\|s_r(t,z)\|^2]  \qquad \qquad \tag{Lemma \ref{Gaussian_integ} for the first term}} \\
       &{=  \frac{\left(\frac{d}{(e^{2t}-1)^{{1}/{m}}}\right)^{-l}}{l(e^{2t}-1)^{l+2}} \cdot (d+2l+4)^{l+2}  -\frac{e^{-2t}}{2} \frac{\left(\frac{d}{(e^{2t}-1)^{{1}/{m}}}\right)^m}{m}\frac{d}{dt}\mathbb{E}_{q_t}[\|s_r(t,z)\|^2]  \qquad \qquad \tag{$c_m=\frac{d}{(e^{2t}-1)^{{1}/{m}}}$}} \\
       &{=  \frac{d^{2}\cdot (1+2l/d+4/d)^{l+2}}{l(e^{2t}-1)^{l+2-l/m}}  -\frac{e^{-2t}}{2}  \frac{d^{m}}{m(e^{2t}-1)}\frac{d}{dt}\mathbb{E}_{q_t}[\|s_r(t,z)\|^2]  } \\
       &{=  \frac{d^2}{(e^{2t}-1)^{3}}\cdot \underbrace{\frac{(1+2\frac{\log d}{d}+\frac{6}{d})^{\log d+3}}{(1+\log d)}}_{C_d}  -\frac{e^{-2t}}{2}  \frac{d^{1+\frac{1}{\log d}}}{(1+\frac{1}{\log d})(e^{2t}-1)}\frac{d}{dt}\mathbb{E}_{q_t}[\|s_r(t,z)\|^2] } \tag{$l=1+\log d$, $m=1+\frac{1}{\log d}$} \\
       &{\leq  \frac{C_dd^2}{(e^{2t}-1)^{3}}    -  \frac{e^{-2t}de}{2(1+\frac{1}{\log d})(e^{2t}-1)}\frac{d}{dt}\mathbb{E}_{q_t}[\|s_r(t,z)\|^2]  } \tag{$d^{\frac{1}{\log d}}=e$} \\
    \end{align*}
    where $C_d =\frac{(1+2\frac{\log d}{d}+\frac{6}{d})^{\log d+3}}{(1+\log d)} $.
\end{proof}

\noindent 
Now, using the previous two lemmas and the Lemma \ref{gen_eq_rhs_terms_bound}, we bound the second order spatial derivative term of the score function in the following lemma.
\begin{lemma}
\label{Laplacian_score}
    We have the following bound for the second order score derivative term in Eq. \ref{eq_overall_to_spatial}:  
    \begin{align*}
       \mathbb{E}_{q_t}\left[\|\Delta s_r(t,z)\|_2^2\right] &\leq \frac{8C_d d^2}{13(e^{2t}-1)^{3}}    -  \frac{de^{-2t}}{(e^{2t}-1)} \frac{d}{dt}\mathbb{E}_{q_t}[\|s_r(t,z)\|^2]  -\frac{8}{13} e^{-2t}\frac{d}{dt}\mathbb{E}_{q_t}\left[\|\nabla s_r(t,z)\|_F^2\right] - \\
       & \qquad \qquad \qquad \frac{6e^{-2t}}{13}\frac{d}{dt}  \mathbb{E}_{q_t}\left[\|s_r(t,z)\|^4_2\right]
    \end{align*}
    where $C_d$ is defined in Lemma \ref{hess_p_grad_s_norm}.
\end{lemma}

\noindent
\textit{Proof} The proof is just a careful utilization of the integration by parts, Fokker Plank Equation (FPE) and the reverse ODE Eq. \ref{rescaled_back_ode}. {The proof starts with manipulating the term: $\frac{d}{dt}\mathbb{E}_{q_t}\left[\|\nabla s_r(t,z)\|_F^2\right]$ 
to break it down in the target term and remaining terms from the previous two lemmas and Lemma \ref{gen_eq_rhs_terms_bound}. Then the target term term is expressed via this term and the remaining terms where we replace the bounds for the remaining terms from the mentioned lemmas.  }
It is as follows (again we use $\partial_i$ for the derivative w.r.t. $i^{th}$ coordinate and the Laplacian by $\sum_i\partial_i\partial_i$):
\begin{align*}
    & \frac{d}{dt} \int q_t(z) \|\nabla s_r(t,z)\|^2_F dz \\
    &= \int \partial_t q_t(z) \|\nabla s_r(t,z)\|^2_F dz +  \int  q_t(z)  \partial_t\sum_{i,j}\partial_j s^2_r(t,z)_i  dz \\
    &= \int e^{2t}\sum_{i=1}^d\partial_i\partial_i q_t(z) \|\nabla s_r(t,z)\|^2_F dz +  \int  q_t(z)  \left(2 \sum_{i,j=1}^d \partial_j s_r(t,z)_i \partial_j \partial_t  s_r(t,z)_i \right) dz \tag{Eq. \ref{fpe_rescaled} for first term}\\
    &\overset{(a)}{=} \int e^{2t} \sum_i \partial_i\partial_i q_t(z) \|\nabla s_r(t,z)\|^2_F dz + 2e^{2t} \int  \underbrace{q_t(z) \sum_{i,j} \partial_j s_r(t,z)_i}_{I_1} \partial_j \left(\sum_k \partial_k\partial_k s_r(t,z)_i + \sum_k \partial_k s^2_r(t,z)_i \right)  dz \\
 &\overset{(b)}{=} \int e^{2t} \sum_i \partial_i\partial_i q_t(z) \|\nabla s_r(t,z)\|^2_F dz - 2e^{2t} \sum_{i,j,k} 
 \int  \partial_j q_t(z) \partial_j s_r(t,z)_i   \left( \partial_k \partial_k s_r(t,z)_i + \partial_k s^2_r(t,z)_i \right)  dz \\ 
 & \qquad \qquad - 2e^{2t} \sum_{i,j,k} 
 \int  q_t (z)  \partial_j \partial_j s_r(t,z)_i  \left( \partial_k \partial_k s_r(t,z)_i + \partial_k s^2_r(t,z)_i \right)  dz \\
 &\overset{(c)}{=} e^{2t} \Bigg(\int  \sum_i \partial_i\partial_i q_t(z) \|\nabla s_r(t,z)\|^2_F dz - 2 \sum_{i} 
 \int  q_t(z) \sum_j s_r(t,z)_j \partial_i s_r(t,z)_j   \sum_k\left( \partial_k \partial_k s_r(t,z)_i + \partial_k s^2_r(t,z)_i \right)  dz \qquad \\ 
 &  \qquad  - 2 \sum_{i} 
 \int  q_t (z) \sum_j \partial_j \partial_j s_r(t,z)_i  \sum_k \left( \partial_k \partial_k s_r(t,z)_i + \partial_k s^2_r(t,z)_i \right)  dz \Bigg) 
\end{align*}
where $(a)$ implies use of Lemma \ref{score_fpe} for the second term, $(b)$ implies using Integration By Parts for the second term where we consider the term $I_1$ as one part and the remaining as other, $(c)$ implies using $\partial_jq_t(z) = q_t(z)s_r(t,z)_j$ and then $\partial_js_r(t,z)_i = \partial_is_r(t,z)_j $ in the second term.
Now, we consider the terms except first, 
treating $\sum_j\partial_is^2_r(t,z)_j=2\sum_j s_r(t,z)_j \partial_i s_r(t,z)_j = b_i$ and $\sum_j \partial_j \partial_j s_r(t,z)_i = a_i$, these terms can be written as:
\begin{align*}
    = \sum_i  - b_i \cdot (a_i+b_i) - 2a_i(a_i+b_i)  = \sum_i - 2a_i^2 -b_i^2 - 3a_ib_i \leq \sum_i - 2a_i^2 -b_i^2 + \frac{3}{8}a^2_i + {6}b^2_i
\end{align*}
which leads us to:
{
\begin{align*}
   e^{-2t} \frac{d}{dt} \mathbb{E}_{q_t}\left[\|\nabla s_r(t,z)\|^2_F\right] &\leq \underbrace{\int  \sum_i \partial_i\partial_i q_t(z) \|\nabla s_r(t,z)\|^2_F dz}_{T_0} - \frac{13}{8} \underbrace{\sum_{i} 
 \int  q_t(z) \left(\sum_j \partial_j \partial_j s_r(t,z)_i \right)^2   dz}_{T_1 ({\text{target term}})} \\
 & \qquad \qquad \qquad +6 \underbrace{\sum_{i} 
 \int  q_t(z) \left(\sum_j s_r(t,z)_j \partial_j s_r(t,z)_i \right)^2   dz}_{T_2} \qquad \\ 
\end{align*}
}
{Denoting the first term in the RHS as $T_0$, second or the target term as $T_1$ and third term as $T_2$, we have the following expression for our target term $T_1$ (rewriting $\sum_i\partial_i\partial_i$ as Laplacian operator):
\begin{align*}
     &\frac{13}{8}\mathbb{E}_{q_t}\left[\|\Delta s_r(t,z)\|_2^2\right] \\
     &\leq T_0+6T_2- e^{-2t}\frac{d}{dt}\mathbb{E}_{q_t}\left[\|\nabla s_r(t,z)\|_F^2\right] \\
     &= T_0+\frac{3}{2}\mathbb{E}_{q_t}\left[\|\nabla \|s_r(t,z)\|^2_2\|_2^2\right]- e^{-2t}\frac{d}{dt}\mathbb{E}_{q_t}\left[\|\nabla s_r(t,z)\|_F^2\right] \tag{rewriting $T_2$}\\
     &\overset{(a)}{\leq}  \frac{4d^2}{(e^{2t}-1)^{3}}    -  \frac{e^{-2t}de}{2(1+\frac{1}{\log d})(e^{2t}-1)}\frac{d}{dt}\mathbb{E}_{q_t}[\|s_r(t,z)\|^2]   -\frac{3e^{-2t}}{4}\frac{d}{dt}  \mathbb{E}_{q_t}\left[\|s_r(t,z)\|^4_2\right]- e^{-2t}\frac{d}{dt}\mathbb{E}_{q_t}\left[\|\nabla s_r(t,z)\|_F^2\right]  \\
\end{align*}
where in step $(a)$ we have just used Lemma \ref{hess_p_grad_s_norm} for $T_0$ term and the  observation that the third term (obtained by rewriting $T_2$) is just the $X'_q$ in Lemma $\ref{gen_eq_rhs_terms_bound}$ for $q=4$ which we bound using the Lemma \ref{gen_eq_rhs_terms_bound}. Now since $d> 1$, we have approximated the value $\frac{4e}{13(1+\frac{1}{\log d})} < 1$ (since $\frac{d}{dt}\mathbb{E}_{q_t}[\|s_r(t,z)\|^2]$ is positive, can be seen from Lemma \ref{gen_eq_rhs_terms_bound}) leading to the final bound. 
}

\subsubsection{Bounding the discretization error for each interval}
\noindent
 Utilising the lemmas discussed above for bounding the spatial derivaitve terms in Eq. \ref{eq_overall_to_spatial}, we now provide a lemma which using these bounds provides a final aggregated bound for $\mathbb{E}_{q_t}[\|s_r'(t,z)\|_2^2]$.
\begin{lemma}
\label{overall_derivative_bound}
    For the rescaled function, we have the following bound for $\mathbb{E}_{q_t}[\|s_r'(t,z)\|_2^2]$:
    \begin{align*}
        \mathbb{E}_{q_t}[\|s_r'(t,z)\|_2^2] \leq { {\frac{40C_d d^2e^{4t}}{13(e^{2t}-1)^{3}}    -  \frac{5e^{2t}de}{4(e^{2t}-1)} \frac{d}{dt}\mathbb{E}_{q_t}[\|s_r(t,z)\|^2]  - \frac{10}{13} e^{2t}\frac{d}{dt}\mathbb{E}_{q_t}\left[\|\nabla s_r(t,z)\|_F^2\right] - e^{2t}\frac{d}{dt}\mathbb{E}_{q_t}[\|s_r(t,z)\|^4]} } 
    \end{align*}
    where $C_d$ is taken from Lemma \ref{hess_p_grad_s_norm}.
\end{lemma}

\begin{proof}
    \begin{align*}
    &\mathbb{E}_{q_t}[\|s_r'(t,z)\|^2] \\
     &= {\mathbb{E}_{q_t}\left[{ e^{4t}\| \Delta s_r(t,z) \|^2_2} + e^{4t}{\left\|\nabla s_r(t,z)^\top s_r(t,z)\right\|_2^2} + 2e^{4t}(\Delta s_r(t,z))^\top\left(\nabla s_r(t,z)^\top s_r(t,z)\right)\right]} \qquad \text{(Lemma \ref{overall_to_spatial})} \\
     &\leq {\mathbb{E}_{q_t}\left[{ \frac{5}{4}e^{4t}\| \Delta s_r(t,z) \|^2_2} + 5e^{4t}{\left\|\nabla s_r(t,z)^\top s_r(t,z)\right\|_2^2}\right]  \tag{$2a\cdot b \leq \frac{||a||^2}{4} + 4\|b\|^2$ for $a,b\in \mathbb{R}^d$ }}\\
     &= { \frac{5}{4}e^{4t}{\mathbb{E}_{q_t}\left[\| \Delta s_r(t,z) \|^2_2\right] + \frac{5}{4}e^{4t}\mathbb{E}_{q_t}\left[{\|\nabla \|s_r(t,z)\|^2_2\|^2_2}\right]} } \\
     &\leq { \frac{5}{4}e^{4t}{\mathbb{E}_{q_t}\left[\| \Delta s_r(t,z) \|^2_2\right] - \frac{5}{16}e^{2t}\frac{d}{dt}\mathbb{E}_{q_t}[\|s_r(t,z)\|^4]} } \tag{using $X'_m$ ($m=4$) bound from Lemma \ref{gen_eq_rhs_terms_bound}}\\
     &\leq { {\frac{40 C_d d^2e^{4t}}{13(e^{2t}-1)^{3}}    -  \frac{5e^{2t}d}{4(e^{2t}-1)} \frac{d}{dt}\mathbb{E}_{q_t}[\|s_r(t,z)\|^2]  - \frac{10}{13} e^{2t}\frac{d}{dt}\mathbb{E}_{q_t}\left[\|\nabla s_r(t,z)\|_F^2\right] - \frac{185}{208}e^{2t}\frac{d}{dt}\mathbb{E}_{q_t}[\|s_r(t,z)\|^4]} } 
\end{align*}
{where the  last inequality uses Lemma \ref{Laplacian_score} for the first term.} Now since $\frac{d}{dt}\mathbb{E}_{q_t}[\|s_r(t,z)\|^4]$ will be negative from Lemma \ref{gen_eq_rhs_terms_bound}, then here also we can use $\frac{185}{208}<1$ leading to the final bound.
\end{proof}

\noindent
Now, we have the following Lemma for bounding the discretization error $z(t)$: $\mathbb{E}\left[\|z_{k-0.5}-\tilde{z}_{k-0.5}\|_2^2\right]$. 
\begin{lemma}
\label{final_discretization_error_bound}
    The discretization error for each interval $\mathbb{E}\left[\|z_{k-0.5}-\tilde{z}_{k-0.5}\|_2^2\right]$ discussed Lemma \ref{overall_w_2_error_decomp} can be bounded as (where $h'_k=h_k+h_{k-1}$ and recall $t_{k-2}=t_k-h_{k}-h_{k-1}$):
    \begin{align*}
       e^{-2t_{k-2}} \mathbb{E}\left[\|z_{k-0.5}-\tilde{z}_{k-0.5}\|_2^2\right] &\leq \frac{\left((2C_d+6)d^2 +16d\right)(h'_k)^3e^{h'_k}(e^{h'_k}-1)}{(1-e^{-2t_{k-2}})^3} \\
         &\quad -\frac{(h'_k)^3}{2}e^{h'_k}\left[e^{4t}\left(\frac{10}{13}\mathbb{E}_{q_t}\left[\|\nabla s_r(t,z)\|_F^2\right] +\mathbb{E}_{q_t}[\|s_r(t,z)\|^4] \right) \right]^{t_k}_{t_{k-2}}\\
        &\qquad \qquad - \frac{5(h'_k)^3e^{h'_k}d}{8(1-e^{-2t_{k-2}})} \left[ {e^{2t}}\mathbb{E}_{q_t}[\|s_r(t,z)\|^2] \right]^{t_k}_{t_{k-2}}
    \end{align*}
    \end{lemma}
    \begin{proof}
    Using Lemma \ref{discretization_error_rescaled} and Lemma \ref{overall_derivative_bound}, it can be bounded as
    \begin{align*}
        &e^{-2t_{k-2}}\mathbb{E}\left[\|z_{k-0.5}-\tilde{z}_{k-0.5}\|_2^2\right] \\
        &\leq e^{-2t_{k-2}} \frac{1}{2}(h_k+h_{k-1})^3  \int_{t_{k-2}}^{t_k}e^{4t}\mathbb{E}\left[\left\|s'_r(t,z(t)) \right\|_2^2\right]dt \\
        &{\leq \frac{(h_k+h_{k-1})^3}{2}e^{h_k+h_{k-1}} \int_{t_{k-2}}^{t_k}e^{2t}\mathbb{E}_{q_t}\left[\left\|s_r(t,z(t)) \right\|_2^2\right]dt}  \\
        &\leq \frac{(h_k+h_{k-1})^3}{2}e^{h_k+h_{k-1}} \\
        &\quad \int_{t_{k-2}}^{t_k}\left({ {\frac{40C_dd^2e^{6t}}{13(e^{2t}-1)^{3}}    -  \frac{5e^{4t}d}{4(e^{2t}-1)} \frac{d}{dt}\mathbb{E}_{q_t}[\|s_r(t,z)\|^2]  - e^{4t}\left(\frac{10}{13}\frac{d}{dt}\mathbb{E}_{q_t}\left[\|\nabla s_r(t,z)\|_F^2\right]+\frac{d}{dt}\mathbb{E}_{q_t}[\|s_r(t,z)\|^4]\right)} } \right)  dt  \\
         & \overset{(c)}{\leq}   \frac{20C_dd^2(h_k+h_{k-1})^3e^{h_k+h_{k-1}}(e^{h_k+h_{k-1}}-1)}{13(1-e^{-2t_{k-2}})^3} \\
         &\qquad -\frac{5(h_k+h_{k-1})^3e^{h_k+h_{k-1}}d}{8(1-e^{-2t_{k-2}})}
         \left(\left[ {e^{2t}}\mathbb{E}_{q_t}[\|s_r(t,z)\|^2] \right]^{t_k}_{t_{k-2}}-\int_{t_{k-2}}^{t_k} {2e^{2t}} \mathbb{E}_{q_t}[\|s_r(t,z)\|^2] dt\right)  \\
         & \qquad\qquad-\frac{(h_k+h_{k-1})^3}{2}e^{h_k+h_{k-1}}\left[e^{4t}\left(\frac{10}{13}\mathbb{E}_{q_t}\left[\|\nabla s_r(t,z)\|_F^2\right] +\mathbb{E}_{q_t}[\|s_r(t,z)\|^4] \right) \right]^{t_k}_{t_{k-2}} \\
         & \quad\quad +(h_k+h_{k-1})^3e^{h_k+h_{k-1}}\int_{t_{k-2}}^{t_k} 2e^{4t}\left(\frac{10}{13}\mathbb{E}_{q_t}\left[\|\nabla s_r(t,z)\|_F^2\right] + \mathbb{E}_{q_t}[\|s_r(t,z)\|^4\right)   dt  \\
         & \overset{(d)}{\leq}  \frac{20C_dd^2(h_k+h_{k-1})^3(e^{h_k+h_{k-1}}-1)}{13(1-e^{-2t_{k-2}})^3} \\
         &\quad + {(h_k+h_{k-1})^3}e^{h_k+h_{k-1}}
         \left(\frac{10d}{8(1-e^{-2t_{k-2}})}\int_{t_{k-2}}^{t_k} {e^{2t}} \frac{d}{e^{2t}-1}dt  + \int_{t_{k-2}}^{t_k} e^{4t}\left(\frac{20(2d^2+6d)}{13(e^{2t}-1)^2} + \frac{2d^2+4d}{(e^{2t}-1)^2}\right)   dt \right)  \\
         & \quad-\frac{(h_k+h_{k-1})^3}{2}e^{h_k+h_{k-1}}\left(\left[e^{4t}\left(\frac{10}{13}\mathbb{E}_{q_t}\left[\|\nabla s_r(t,z)\|_F^2\right] +\mathbb{E}_{q_t}[\|s_r(t,z)\|^4] \right) \right]^{t_k}_{t_{k-2}}   \right) 
         \\& \qquad - \frac{5(h_k+h_{k-1})^3e^{h_k+h_{k-1}}d}{8(1-e^{-2t_{k-2}})} \left[ {e^{2t}}\mathbb{E}_{q_t}[\|s_r(t,z)\|^2] \right]^{t_k}_{t_{k-2}} \\
         &\leq \frac{\left((2C_d+6)d^2 +16d\right)(h_k+h_{k-1})^3e^{h_k+h_{k-1}}(e^{h_k+h_{k-1}}-1)}{(1-e^{-2t_{k-2}})^3} \\
         &\quad -\frac{(h_k+h_{k-1})^3}{2}e^{h_k+h_{k-1}}\left[e^{4t}\left(\frac{10}{13}\mathbb{E}_{q_t}\left[\|\nabla s_r(t,z)\|_F^2\right] +\mathbb{E}_{q_t}[\|s_r(t,z)\|^4] \right) \right]^{t_k}_{t_{k-2}}\\
        &\qquad \qquad - \frac{5(h_k+h_{k-1})^3e^{h_k+h_{k-1}}d}{8(1-e^{-2t_{k-2}})} \left[ {e^{2t}}\mathbb{E}_{q_t}[\|s_r(t,z)\|^2] \right]^{t_k}_{t_{k-2}}
    \end{align*}
    where $(c)$  uses $\int\frac{e^{6t}}{(e^{2t}-1)^3}dt \leq \frac{e^{4t_{k-2}}}{e^{4t_{k-2}}-1}\int^{t_k}_{t_{k-2}}\frac{e^{2t}}{e^{2t}-1}dt$, $ \int^{t_k}_{t_{k-2}}\frac{e^{2t}}{e^{2t}-1}dt = \frac{1}{2}\log\left(\frac{e^{2t_k}-1}{e^{2t_{k-2}}-1}\right)$, $\log(1+x)\leq x$ for the first term  and implies applying the Integration By Parts to the second and third terms, where in the second term, we have considered $e^{2t}\frac{d}{dt}\mathbb{E}_{q_t}[\|s_r(t,z)\|^2]$ as one term and use the max value of the remaining term since we know from Lemma \ref{spatial_time_relation} that $e^{2t}\frac{d}{dt}\mathbb{E}_{q_t}[\|s_r(t,z)\|^2] \leq 0$, in step $(d)$we recollect the integral terms and since they have a positive contribution, just replace the term inside the integral with the upper bound from Lemma \ref{score_bound}. In the last step, similar to step $(c)$, we have used $ \int^{t_k}_{t_{k-2}}\frac{e^{2t}}{e^{2t}-1}dt = \frac{1}{2}\log\left(\frac{e^{2t_k}-1}{e^{2t_{k-2}}-1}\right)$, for $\int^{t_k}_{t_{k-2}}\frac{e^{4t}}{(e^{2t}-1)^2}dt \leq \frac{e^{2t_{k-2}}}{e^{2t_{k-2}}-1}\int^{t_k}_{t_{k-2}}\frac{e^{2t}}{e^{2t}-1}dt$ and finally $\log(1+x)\leq x$ . 

\end{proof}

\subsection{Proving Theorem \ref{main_theorem_non_smooth}}
\label{final_theorem_proof}

\noindent
We first discuss a Lemma based on standard calculus which would be utilized in the Theorem proof.

\begin{lemma}\label{utility_lemma}
Fix $c\in(0,\tfrac12)$ and set $a:=1-c\in(\tfrac12,1)$ and $b:=2-c\in(\tfrac32,2)$. For $x\in(0,1)$ define
\[
f(x):=\frac{(2-c)^3\,x^3\,e^{bx}}{e^{ax}-1}, 
\qquad
g(x):=\frac{(2-c)^3\,x^3\,e^{bx}}{(e^{ax}-1)\,\bigl(1-e^{-\frac{a^2}{c}x}\bigr)}.
\]
Then $f$ and $g$ are increasing on $(0,1)$ and there exist absolute constants $C_f,C_g>0$ (independent of $c$ and $x$) such that for all $x\in(0,1)$,
\[
0\le f(x)-f\!\big((1-c)^2x\big)\le C_f\,c\,x^2,
\qquad
0\le g(x)-g\!\big((1-c)^2x\big)\le C_g\,\frac{c}{1-e^{-\frac{(1-c)^2}{c}x}}\,x^2.
\]
\end{lemma}

\begin{proof}
Using $\frac{d}{dx}\log(e^{\alpha x}-1)=\frac{\alpha e^{\alpha x}}{e^{\alpha x}-1}$,
\[
(\log f)'(x)=\frac{3}{x}+b-\frac{a e^{ax}}{e^{ax}-1}, 
\qquad
(\log g)'(x)=\frac{3}{x}+b-\frac{a e^{ax}}{e^{ax}-1}-\frac{\frac{a^2}{c}}{e^{\frac{a^2}{c}x}-1}.
\]
For $t>0$ we have the elementary bound $\frac{1}{e^{t}-1}\le \frac{1}{t}$; hence
\[
\frac{a e^{ax}}{e^{ax}-1}\le a+\frac1x\le 1+\frac1x, 
\qquad
\frac{\frac{a^2}{c}}{e^{\frac{a^2}{c}x}-1}\le \frac{1}{x}.
\]
Therefore, for $x\in(0,1)$,
\[
(\log f)'(x)\;\ge\;\frac{3}{x}+b-\Bigl(1+\frac1x\Bigr)
= \frac{2}{x}+(1-c)\;\ge\;\frac{2}{x}+\frac12\;>\;0,
\]
and
\[
(\log g)'(x)\;\ge\;\frac{3}{x}+b-\Bigl(1+\frac1x\Bigr)-\frac1x
= \frac{1}{x}+(1-c)\;\ge\;\frac{1}{x}+\frac12\;>\;0.
\]
Hence $f$ and $g$ are increasing on $(0,1)$. Using $e^{ax}-1\ge ax$ and $(2-c)^3\le 8$, $b\le 2$, we get for $x\in(0,1)$
\[
f(x)=\frac{(2-c)^3 x^3 e^{bx}}{e^{ax}-1}
\le \frac{8\,x^3\,e^{2}}{ax}\le 16e^{2}\,x^2.
\]
Consequently,
\[
g(x)=\frac{f(x)}{1-e^{-\frac{a^2}{c}x}}
\le \frac{16e^2\,x^2}{1-e^{-\frac{a^2}{c}x}}
\]
From above, we have:
\[
f'(x)=f(x)\,(\log f)'(x)\le f(x)\Bigl(\frac{3}{x}+b\Bigr)
\le 16e^2 x^2\Bigl(\frac{3}{x}+2\Bigr)
\le 80e^2x.
\]
Similarly,
\[
g'(x)=g(x)\,(\log g)'(x)
\le g(x)\Bigl(\frac{3}{x}+b\Bigr)
\le 16e^2\,\frac{x^2}{1-e^{-\frac{a^2}{c}x}}\Bigl(\frac{3}{x}+2\Bigr)
\le 80e^2\,\frac{x}{1-e^{-\frac{a^2}{c}x}},
\]

\noindent
Let $y:=(1-c)^2x\in(0,x)$. Then $x-y=\bigl(1-(1-c)^2\bigr)x=(2c-c^2)x\le 2cx$.
By the mean value theorem, for some $\xi\in(y,x)\subset(0,1)$,
\[
f(x)-f(y)=f'(\xi)\,(x-y)\le 80e^2\,\xi\,(2cx)\le 160e^2\,c\,x^2,
\]
which yields $f(x)-f((1-c)^2x)\le C_f\,c\,x^2$ where $C_f=160e^2$ is an absolute constant. Likewise, for some $\eta\in(y,x)$,
\[
g(x)-g(y)=g'(\eta)\,(x-y)\le 80e^2\,\frac{\eta}{1-e^{-\frac{a^2}{c}\eta}}\,(2cx)
\le 160e^2\,\frac{c}{1-e^{-\frac{a^2}{c}x}}\,x^2,
\]
since $t\mapsto 1-e^{-t}$ is increasing and $\eta\le x$. This gives
\[
g(x)-g\big((1-c)^2x\big)\le C_g\,\frac{c}{1-e^{-\frac{(1-c)^2}{c}x}}\,x^2
\]
where $C_g=160e^2$ is an absolute constant. \\
\end{proof}

\noindent
\paragraph{Proof of Theorem \ref{main_theorem_non_smooth}.}Now, using the Lemmas discussed above, we provide the proof for Theorem \ref{main_theorem_non_smooth}.


\begin{proof}
We first bound the discretization error term  for the rescaled process in Lemma \ref{final_discretization_error_bound} aggregated across all the intervals. For this, using $h'_k = h_{k} + h_{k-1}$, we bound it as following:
\begin{align*}
    &\sum^{K+1}_{k=2} \frac{e^{-2t_{k-2}}}{e^{2h_{k-1}}-1}  \mathbb{E}\left[\|z_{k-0.5}-\tilde{z}_{k-0.5}\|_2^2\right] \\
    &\overset{}{\leq}  \sum_{k=2}^{K+1}\frac{\left((2C_d+6)d^2 +16d\right)(h'_k)^3e^{h'_k}(e^{h'_k}-1)}{(e^{2h_{k-1}}-1)(1-e^{-2t_{k-2}})^3} \\
        &\quad +\sum_{k=2}^{K+1}\frac{(h'_k)^3e^{h'_k}}{2(e^{2h_{k-1}}-1)}\left[ e^{4t}\left(\frac{10}{13}\mathbb{E}_{q_t}\left[\|\nabla s_r(t,z)\|_F^2\right] +\mathbb{E}_{q_t}[\|s_r(t,z)\|^4] \right) + \frac{5de^{2t}}{4(1-e^{-2t_{k-2}})}\mathbb{E}_{q_t}\left[\|s_r(t,z)\|^2\right] \right]^{t_{k-2}}_{t_k} \\
        &= \sum_{k=2}^{K+1}\frac{\left((2C_d+6)d^2 +16d\right)(h'_k)^3e^{h'_k}(e^{h'_k}-1)}{(e^{2h_{k-1}}-1)(1-e^{-2t_{k-2}})^3} \\
        &\qquad +\frac{(h'_2)^3e^{h'_2}}{e^{2h_1}-1}\,R(t_0)
+\frac{(h'_3)^3e^{h'_3}}{e^{2h_2}-1}\,R(t_1)+\frac{(h'_2)^3e^{h'_2}}{(e^{2h_1}-1)(1-e^{-2t_0})}\,R_1(t_0)
+\frac{(h'_3)^3e^{h'_3}}{(e^{2h_2}-1)(1-e^{-2t_1})}\,R_1(t_1)\\
& \qquad +\sum_{k=2}^{K-1}\left(\frac{(h'_{k+2})^3e^{h'_{k+2}}}{e^{2h_{k+1}}-1}-\frac{(h'_k)^3e^{h'_k}}{(e^{2h_{k-1}}-1)}\right)\!R(t_k)
-\frac{(h'_{K+1})^3}{(e^{2h_{K}}-1)}\,R(t_{K+1})
-\frac{(h'_K)^3}{e^{2h_{K-1}}-1}\,R(t_{K})  \\
& \qquad +\sum_{k=2}^{K-1}\left(\frac{(h'_{k+2})^3e^{h'_{k+2}}}{(e^{2h_{k+1}}-1)(1-e^{-2t_k})}-\frac{(h'_k)^3e^{h'_k}}{(e^{2h_{k-1}}-1)(1-e^{-2t_{k-2}})}\right)\!R_1(t_k) \\
& \qquad -\frac{(h'_{K+1})^3e^{h'_{K+1}}}{(e^{2h_{K}}-1)(1-e^{-2t_{K-1}})}\,R_1(t_{K+1})
-\frac{(h'_K)^3e^{h'_K}}{(e^{2h_{K-1}}-1)(1-e^{-2t_{K-2}})}\,R_1(t_{K})  \\
\end{align*}
where $R(t)=\frac{1}{2}e^{4{t}}\left(\frac{10}{13}\mathbb{E}_{q_{t}}\left[\|\nabla s_r(t,z)\|_F^2\right] +\mathbb{E}_{q_{t}}[\|s_r({t},z)\|^4] \right) \geq 0 $,   $ R_1(t)=\frac{5de^{2t}}{8}\mathbb{E}_{q_{t}}[\|s_r(t,z)\|^2]\geq 0$ . \\

\paragraph{Selecting the step size.} Now for the mentioned choice of the step size $h_k=t_k-t_{k-1}=c\min\{1,t_k\}$, we will have $t_{k-1}=(1-c)t_k,\, h_{k-1}=(1-c)h_{k}$ when $t_k\leq 1$ and $h_k=c$ for remaining. Since $t_0=\delta$, we will have:
    \begin{align*}
        \delta = (1-c)^M \, ; \quad T-1 = c(K+1-M)
    \end{align*}
    for some $M \leq K+2$  with $t_M=1$. Thus, we will have $c \lesssim \frac{\log\left(\frac{1}{\delta}\right) +T}{K}$ and will have a very small value for the mentioned condition $K \geq d(\frac{1}{\delta} + T)$. Also, for the coefficients of terms containing $R$, for $t_k\leq 1$ we will have $t_{k-1}=(1-c){t_k}, h_{k-1}=(1-c)h_k, h_{k-2}=(1-c)^2h_{k}$ and thus, we will have: \\ $\frac{(h_{k+2}+h_{k+1})^3e^{h_{k+2}+h_{k+1}}}{e^{h_{k+1}}-1}-\frac{(h_k+h_{k-1})^3e^{h_k+h_{k-1}}}{e^{2h_{k-1}}-1} = \left(\frac{(2-c)^3h^3_{k+2}e^{(2-c)h_{k+2}}}{e^{(1-c)h_{k+2}}-1}-\frac{(2-c)^3h^3_{k}e^{(2-c)h_k}}{e^{(1-c)h_{k}}-1}\right)$ for $k$ when $t_{k+2}\leq 1$ and $0$ for the rest. This can be written as $f(h_{k+2})-f({h_k})$ where $f(x)=\frac{(2-c)^3x^3e^{(2-c)x}}{e^{({1-c})x}-1}$ would be an increasing function w.r.t. $x$ for $x<1$ in the small $c$ region ($c<0.5$). For this, we will also have $f(x)-f((1-c)^2x) \lesssim cx^2$ (Lemma \ref{utility_lemma}). Similarly for the $R_1$, we have to consider:
    $g(x)= \frac{(2-c)^3x^3e^{(2-c)x}}{(e^{(1-c)x}-1)\left(1-e^{-\frac{(1-c)^2x}{c}}\right)}$ and it will also be increasing on  $(0,1)$ for small $c$ ($c<0.5$) and $g(x)-g((1-c)^2x) \lesssim \frac{c}{1-e^{-\frac{(1-c)^2x}{c}}}x^2$ from Lemma \ref{utility_lemma}.
    Since $h_k$(>0) is an increasing sequence, we can use the upper bound for $R(t)$, $R_1(t)$ using Lemma \ref{score_bound} as:
\begin{align*}
    R(t) \leq \frac{2d^2+5d}{(1-e^{-2t})^2} \, ; \qquad  R_1(t) \leq  \frac{5d^2}{8(1-e^{-2t})}
\end{align*}
Also, the term $C_d$ from Lemma \ref{hess_p_grad_s_norm} is $C_d =\frac{(1+2\frac{\log d}{d}+\frac{6}{d})^{\log d+3}}{(1+\log d)} \leq 12$ for $d\geq 10$ and thus we will have $C_d$ as $O(1)$.
Since $R(t),R_1(t)\geq 0$, the negative terms corresponding to $R({t_{K}}), R(t_{K+1}), R1({t_{K}}), R1(t_{K+1})$ can be dropped and we will finally have:
    \begin{align*}
&\sum^{K+1}_{k=2} \frac{e^{-2t_{k-2}}}{e^{2h_{k-1}}-1}  \mathbb{E}\left[\|z_{k-0.5}-\tilde{z}_{k-0.5}\|_2^2\right] \\
        &\quad \leq \sum_{k=2}^{K+1}\frac{\left((2C_d+6)d^2 +16d\right)(h'_k)^3e^{h'_k}(e^{h'_k}-1)}{(e^{2h_{k-1}}-1)(1-e^{-2t_{k-2}})^3}  \\
        &\qquad +\frac{(h'_2)^3e^{h'_2}}{e^{2h_1}-1}\,R(t_0)
+\frac{(h'_3)^3e^{h'_3}}{e^{2h_2}-1}\,R(t_1)+\frac{(h'_2)^3e^{h'_2}}{(e^{2h_1}-1)(1-e^{-2t_0})}\,R_1(t_0)
+\frac{(h'_3)^3e^{h'_3}}{(e^{2h_2}-1)(1-e^{-2t_1})}\,R_1(t_1)\\
& \qquad +\sum_{k=2}^{K-1}\left(\frac{(h'_{k+2})^3e^{h'_{k+2}}}{e^{2h_{k+1}}-1}-\frac{(h'_k)^3e^{h'_k}}{(e^{2h_{k-1}}-1)}\right)\!R(t_k) \\
& \qquad +\sum_{k=2}^{K-1}\left(\frac{(h'_{k+2})^3e^{h'_{k+2}}}{(e^{2h_{k+1}}-1)(1-e^{-2t_k})}-\frac{(h'_k)^3e^{h'_k}}{(e^{2h_{k-1}}-1)(1-e^{-2t_{k-2}})}\right)\!R_1(t_k) \\
 &\quad \lesssim  \sum_{k=2}^{K+1}\frac{d^2h_k^3}{(1-e^{-2t_{k-2}})^3} +{h^2_2}\,\left(R(t_0) +\frac{R_1(t_0)}{1-e^{-2t_0}}  \right)+ {h^2_3}\,\left(R(t_1) +\frac{R_1(t_1)}{1-e^{-2t_1}}  \right) \\
& \qquad \qquad \;+\;\sum_{k=2}^{M}{ch_{k+2}^2}\left(R(t_k) +\frac{R_1(t_k)}{1-e^{-2t_k}} \right) + \sum_{k=M+1}^{K-1}\frac{c^3R_1(t_k)}{(1-e^{-2t_{k-2}})^2}\\
 & \quad \overset{}{\lesssim} \sum_{k=2}^M \frac{d^2c^3t_k^3}{(1-e^{-2t_{k-2}})^3} + \sum_{k=M+1}^{K+1}\frac{d^2c^3}{(1-e^{-2t_{k-2}})^3} + \frac{c^2t_2^2}{(1-e^{-2t_0})^2} + \frac{c^2t_3^2}{(1-e^{-2t_1})^2} \\
 &\quad \lesssim \sum^{K+1}_{k=2}d^2c^3
    \end{align*}
    where $t_k\leq 1$ for $k\leq M$. Now using Lemma \ref{lemma_second_KL_data_processing_ineq}, Lemma \ref{lemma_first_kl_bw_conditionals}, Lemma \ref{overall_w_2_error_decomp} and the scaling back the above bound on the aggregated error for the rescaled process $\tilde{z}$, we will have (using $u<e^{u}-1<2u$ for $u\in (0,1)$):
    \begin{align*}
        &\KL{p_{t_1}}{\hat{p}_{t_1}}\\
        &\leq  \KL{{p}_{t_{K+1}}}{\hat{p}_{t_{K+1}}} + 
   \E_{{p}_{t_1,..,t_{K+1}}} \left[ \sum_{k=2}^{K+1}
       \KL{{p}_{t_{k-1}|t_k}(\cdot|{x}_{k})}{\hat{p}_{t_{k-1}|t_k}(\cdot|{x}_k)}\right] \\
       &= \KL{{p}_{t_{K+1}}}{\hat{p}_{t_{K+1}}} + \sum_{k=2}^{K+1} {\frac{e^{-2h_{k-1}}}{1-e^{-2h_{k-1}}}} \mathbb{E}\|
      x_{k-0.5}-\tilde{x}_{k-0.5}\|_2^2 \\
      &\qquad \quad + {\frac{e^{-2h_{k-1}}}{1-e^{-2h_{k-1}}}}(e^{h_k+h_{k-1}}-1)^2\mathbb{E}[\|s(t_k,x_{k})- \hat{s}(t_k, x_{k}\|^2]\\
      &\lesssim \KL{{p}_{t_{K+1}}}{\hat{p}_{t_{K+1}}} + \sum_{k=2}^{K+1} d^2c^3 + \sum_{k=2}^{K+1}h_k \mathbb{E}[\|s(t_k,x_{k})- \hat{s}(t_k, x_{k}\|^2]
    \end{align*}    
The last term can be just bounded using Assumption \ref{assumption_score_est} and the first term is the initialization error discussed below.  \paragraph{Initialization Error.} The term $ \KL{{p}_{t_{K+1}}}{\hat{p}_{t_{K+1}}}$ is the error due to initializing the generation using the Normal distribution and can be bounded via convergence of the forward OU process after the total time $T$ \cite{chen2022sampling}:
    \begin{align*}
       \KL{{p}_{t_{K+1}}}{\hat{p}_{t_{K+1}}} \leq (d+m_2)e^{-T}
    \end{align*}
    where $m_2 = \mathbb{E}[\|x_0\|^2]$. Thus, we have the final expression as:
    \begin{align*}
         \KL{p_{t_1}}{\hat{p}_{t_1}}
         &\lesssim (d+m_2)e^{-T} + d^2c^3K + T \varepsilon^2_{score} 
    \end{align*}
\end{proof}

\end{document}